\documentclass[sigconf]{acmart}

\usepackage[utf8]{inputenc} 
\usepackage[T1]{fontenc}    
\usepackage{url}            
\usepackage{booktabs}       
\usepackage{amsfonts}       
\usepackage{nicefrac}       
\usepackage{microtype}      
\usepackage[page,title]{appendix}
\usepackage{comment}
\usepackage{enumitem}

\usepackage{bm}
\usepackage{dsfont}
\usepackage{amsmath}
\usepackage{amsthm}
\newtheorem{theorem}{Theorem}

\usepackage{graphicx}
\usepackage{subfigure}
\usepackage{capt-of}

\usepackage{hyperref}
\usepackage{comment}

\usepackage{wrapfig}       
\usepackage{diagbox}
\usepackage{multirow}
\usepackage{makecell}

\usepackage{pifont}
\newcommand{\cmark}{\ding{51}}%
\newcommand{\xmark}{\ding{55}}%
\newcommand{\Tr}[1]{\text{\normalfont Tr}\left(#1\right)}
\newcommand{\E}[1]{\text{\normalfont E}\left[#1\right]}
\newcommand{\Var}[1]{\text{\normalfont Var}\left(#1\right)}
\newcommand{\norm}[1]{\left\lVert#1\right\rVert}
\newcommand{\Cov}[1]{\text{\normalfont Cov}\left(#1\right)}

\newcommand{\X}{\mathbb{X}}
\newcommand{\Y}{\mathbb{Y}}
\newcommand{\Z}{\mathbb{Z}}
\newcommand{\e}{\bm e}
\newcommand{\x}{\bm x}
\newcommand{\z}{\bm z}
\newcommand{\InterVar}{\normalfont \text{InterVar}}
\newcommand{\IntraVar}{\normalfont \text{IntraVar}}

\newcommand{\paragraphtitle}[1]{\vskip 0.3 em \noindent\textbf{#1}}

\AtBeginDocument{%
  \providecommand\BibTeX{{%
    \normalfont B\kern-0.5em{\scshape i\kern-0.25em b}\kern-0.8em\TeX}}}

\copyrightyear{2021}
\acmYear{2021}
\setcopyright{acmlicensed}
\acmDOI{10.1145/3447548.3467413}

\acmConference[KDD '21]{Proceedings of the 27th ACM SIGKDD Conference on Knowledge Discovery and Data Mining}{August 14--18, 2021}{Virtual Event, Singapore}
\acmBooktitle{Proceedings of the 27th ACM SIGKDD Conference on Knowledge Discovery and Data Mining (KDD '21), August 14--18, 2021, Virtual Event, Singapore}
\acmPrice{15.00}
\acmISBN{978-1-4503-8332-5/21/08}
\acmSubmissionID{rst3097}

\begin{document}
\fancyhead{}

\title{Dimensionwise Separable 2-D Graph Convolution for Unsupervised and Semi-Supervised Learning on Graphs}

\author{Qimai Li}
\authornote{These authors contributed equally.}
\email{csqmli@comp.polyu.edu.hk}
\orcid{0000-0002-6705-7939}
\affiliation{%
  \institution{The Hong Kong Polytechnic University}
  \country{}
}

\author{Xiaotong Zhang}
\authornotemark[1]
\email{zhangxt@dlut.edu.cn}
\affiliation{%
  \institution{School of Software \\ Dalian University of Technology}
  \country{}
}

\author{Han Liu}
\authornotemark[1]
\email{hanliu@dlut.edu.cn}
\affiliation{%
  \institution{School of Software \\ Dalian University of Technology}
  \country{}
}

\author{Quanyu Dai}
\email{daiquanyu@huawei.com}
\affiliation{%
  \institution{Huawei Noah's Ark Lab}
  \country{}
}

\author{Xiao-Ming Wu}
\authornote{Corresponding Author}
\email{xiao-ming.wu@polyu.edu.hk}
\affiliation{%
  \institution{The Hong Kong Polytechnic University}
  \country{}
}

\begin{abstract}

Graph convolutional neural networks (GCN) have been the model of choice for graph representation learning, which is mainly due to the effective design of graph convolution that computes the
representation of a node by aggregating those of its neighbors. However, existing GCN variants commonly use 1-D graph convolution that solely operates on the object link graph without exploring informative relational information among object attributes. This significantly limits their modeling capability and may lead to inferior performance on noisy and sparse real-world networks. In this paper, we explore 2-D graph convolution to jointly model object links and attribute relations for graph representation learning. Specifically, we propose a computationally efficient dimensionwise separable 2-D graph convolution (DSGC) for filtering node features. Theoretically, we show that DSGC can reduce intra-class variance of node features on both the object dimension and the attribute dimension to learn more effective representations. Empirically, we demonstrate that by modeling attribute relations, DSGC achieves significant performance gain over state-of-the-art methods for node classification and clustering on a variety of real-world networks. The source code for reproducing the experimental results is available at \url{https://github.com/liqimai/DSGC}.

\end{abstract}




\maketitle
\sloppy

\section{Introduction}



Graph-structured data occurs in citation networks \cite{yang2015network}, social networks \cite{ye2018deep}, traffic networks\cite{DBLP:conf/kdd/WuZWP20}, protein networks\cite{hamilton2017inductive}, knowledge graphs \cite{DBLP:conf/kdd/WangLJLF20,DBLP:conf/kdd/Pei0Y020} and many other application fields. Analyzing graph data often leads to new insights and interesting discoveries. For example, unsupervised learning on graphs enables to detect user communities in social networks, while semi-supervised learning on graphs helps to predict protein functions in protein networks. Learning effective node representations by utilizing graph structures and other aspects of information such as node content has proved very useful for unsupervised and semi-supervised learning on graphs.




\begin{figure}
    \centering
        \includegraphics[width=0.23\textwidth]{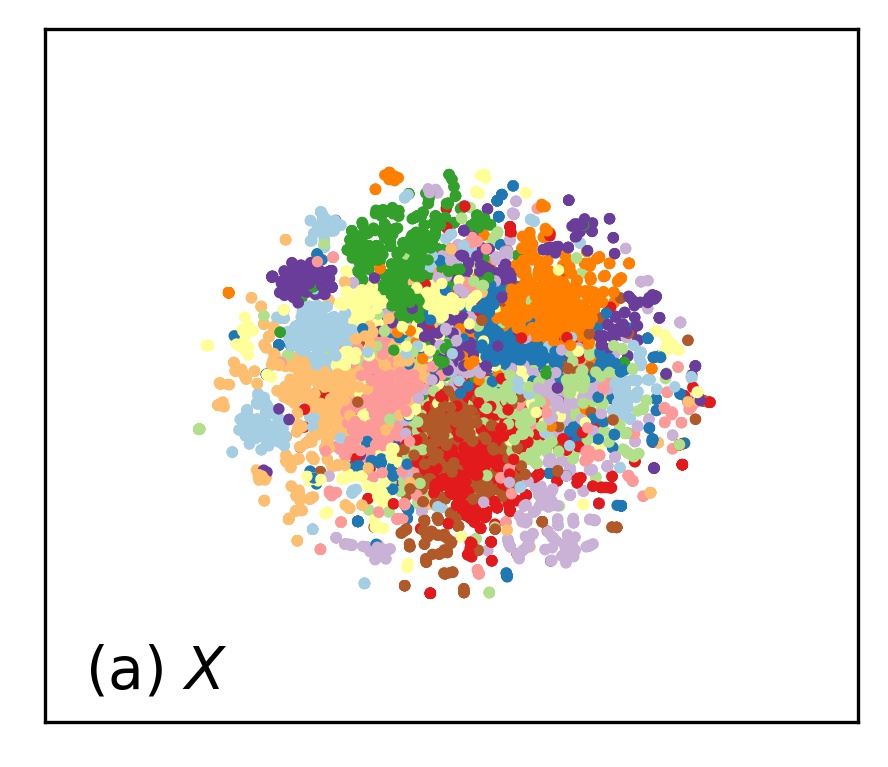}
        \includegraphics[width=0.23\textwidth]{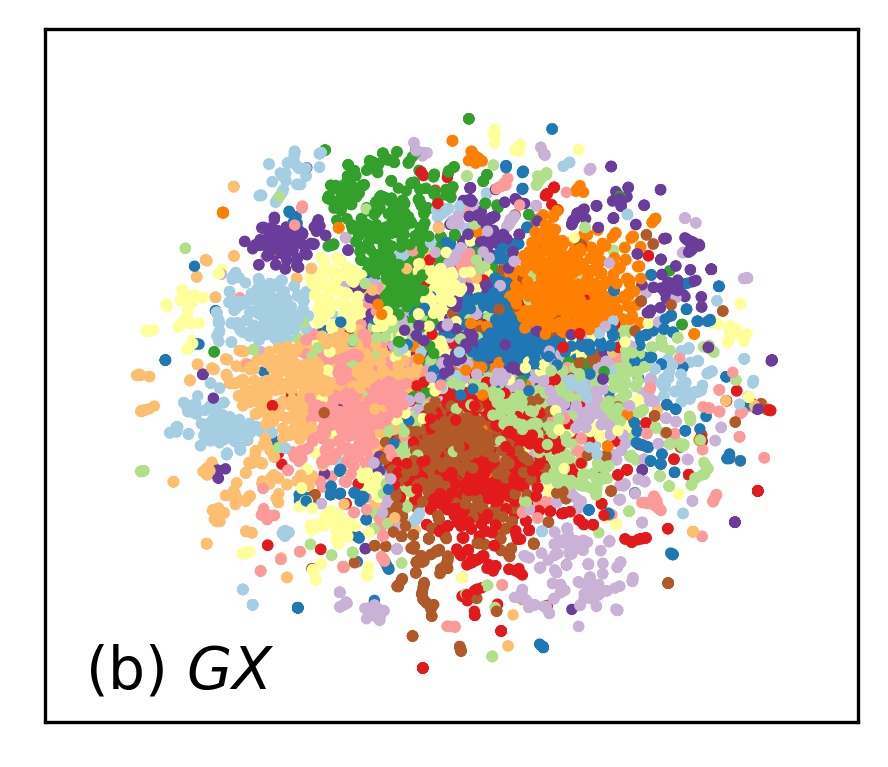}\\
        \includegraphics[width=0.23\textwidth]{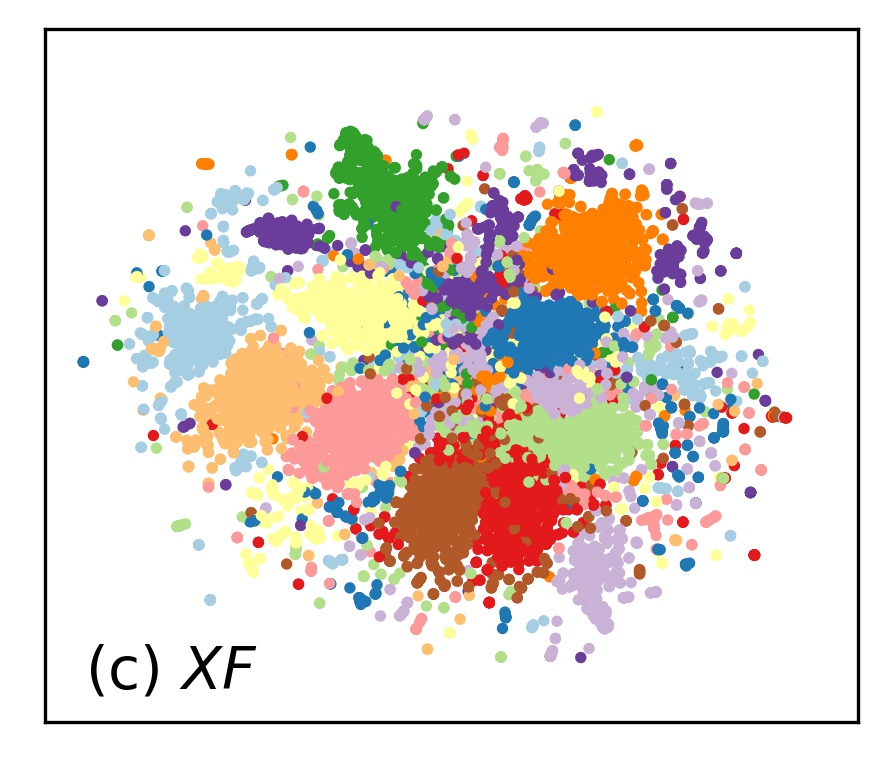}
        \includegraphics[width=0.23\textwidth]{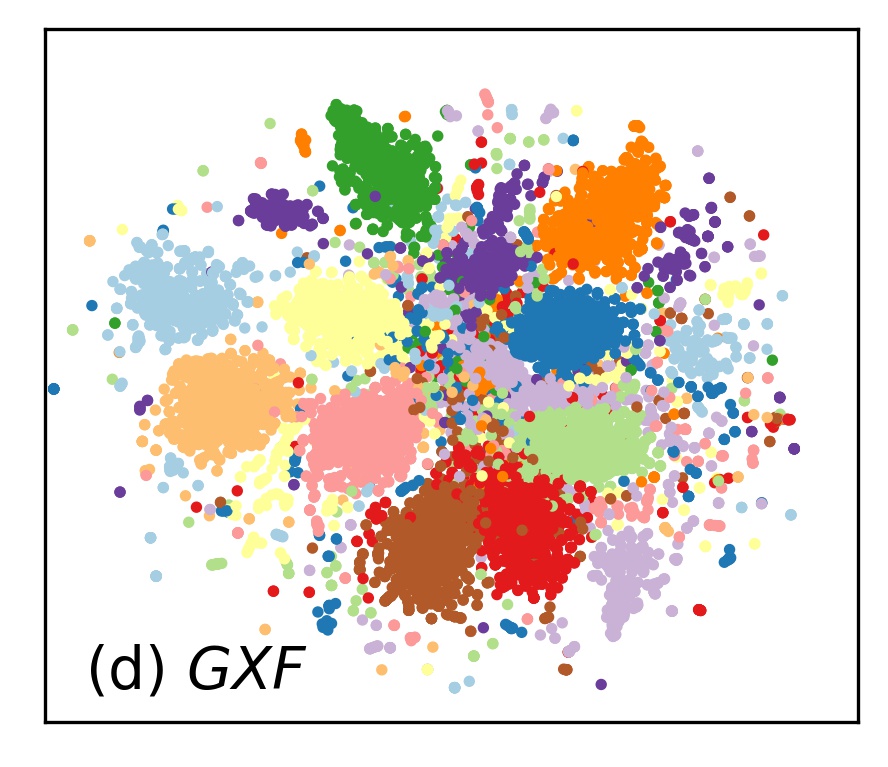}
    \caption{t-SNE visualization of the ``20 Newsgroups'' dataset. (a) Raw features; (b) Filtered by the regular object graph convolution; (c) Filtered by our proposed attribute graph convolution; (d) Filtered by our proposed dimensionwise separable graph convolution (DSGC).}
    \label{fig:visualization}
    \vspace{-0.15in}
\end{figure}

Previous approaches for unsupervised learning have explored applying non-negative matrix factorization, random walk statistics, and Laplacian eigenmaps on both graph structures and node attributes \cite{PanWZZW16,huang2017accelerated} to learn node representations. For semi-supervised learning, a common way is to regularize a supervised classifier trained with node features by a Laplacian regularizer or an embedding-based regularizer to take into account graph structures, e.g., manifold regularization \cite{belkin2006manifold}, manifold denoising \cite{Hein07}, deep semi-supervised embedding \cite{weston2012deep}, and Planetoid \cite{yang2016revisiting}. However, these methods model node connectivity and node content separately and hence may not be able to fully utilize these information.


In the past few years, graph convolutional neural networks (GCN) and variants have dominated the research of graph representation learning and achieved new state-of-the-art results in various learning tasks on graphs, especially in unsupervised learning \cite{wang2017mgae,ijcai/0003LLW19} and semi-supervised learning \cite{kipf2016semi,hamilton2017inductive,velickovic2017graph,li2019label}. The success is mainly attributed to graph convolution, a function that naturally combines node connectivity and node content for feature propagation and smoothing, which computes the representation of a node by aggregating the features of its neighbors. 
The effective utilization of both modalities of data gives the unique advantage to GCN-based methods over previous approaches, including topology-only models~\cite{perozzi2014deepwalk,grover2016node2vec} and graph regularization based methods~\cite{belkin2006manifold,yang2016revisiting}.

In spite of their empirical success, one major limitation of most existing GCN-based methods is that they commonly adopt
one-dimensional (1-D) graph convolution that operates on the \emph{object link graph} to model node (object) relations and features, whose performance critically relies on the quality of the graph. However, real-life networks are often noisy and sparse. For example, in a web graph such as Wikipedia, a hyperlink between two webpages does not necessarily indicate that they belong to the same category, and mixing their features could be harmful for webpage classification or clustering. Moreover, it has been shown that many real-world networks are scale-free and there exist many low-degree nodes \cite{albert2002statistical}. Since these nodes may have very few or even no links to other nodes, it is difficult, and even impossible, to do feature propagation to endow them with similar features as other same-class nodes to facilitate downstream learning tasks.


To address the above limitation, we propose to explore data relations in another dimension, by constructing an \emph{attribute affinity graph}, which encodes relations between object attributes. The underlying assumption is attributes that indicate the same category should have strong relations. For example, in a citation network, object attributes are words, and documents of AI category usually contain words such as ``learning'', ``robotics'', ``machine'', ``neural'', etc. These indicative words for AI category should be more closely related than other non-indicative words. These informative relations can then be utilized for feature smoothing, similar to the use of object links. Importantly, the attribute affinity graph can be a useful complement to the object link graph in learning node representations. For instance, consider a document that has no links to others (e.g., node e or d in Figure~\ref{fig:GXF}), and hence it is impossible to do feature propagation with object links. Yet, with meaningful attribute relations, feature propagation can still be performed along the attribute dimension, and it is possible to obtain similar features for same-class documents.

To formalize the above insight, we propose to perform graph convolution on the attribute affinity graph (Figure~\ref{fig:visualization} (c)). Further, we develop an efficient two-dimensional (2-D) graph convolution to perform feature smoothing on both the object link graph and the attribute affinity graph to learn more effective node representations (Figure~\ref{fig:visualization} (d)). Our main contributions are described as follows.
\begin{itemize}[leftmargin=*]
	\item \textbf{Methodology}: We propose to use 2-D graph convolution to jointly model object links, attribute relations, and object features for node representation learning. Furthermore, we develop a computationally efficient dimensionwise separable 2-D graph convolutional filter (DSGC), which is equivalent to performing 1-D graph convolution alternately on the object dimension and the attribute dimension, as illustrated in Figure~\ref{fig:GXF}. Finally, we propose two learning frameworks based on DSGC for unsupervised node clustering and semi-supervised node classification.

	\item \textbf{Theoretical insight}: We show that the regular 1-D graph convolution on the object link graph can reduce intra-class variance of node features, which helps to explain the success of many existing methods. Further, we show that the same can be proved for graph convolution on a properly constructed attribute affinity graph. Jointly, they provide a theoretical justification of DSGC.

	\item \textbf{Empirical study}: We implement DSGC for semi-supervised node classification and unsupervised node clustering, and conduct extensive experiment on a variety of real-world networks including email networks, citation networks, and web graphs. The comparison with state-of-the-art methods demonstrate the advantages of DSGC over the regular 1-D graph convolution. Moreover, we show that DSGC can be easily plugged into some strong GCN-based methods to further improve their performance substantially.

\end{itemize}

\section{Related Works}

In this work, we focus on the studies of graph convolutional neural networks for unsupervised and semi-supervised learning. There are a large number of prior works for unsupervised and semi-supervised learning on graph-structured data. We can only briefly review a small portion of them due to space limitation. We first review non-GCN-based methods, then review GCN-based methods. 

\paragraphtitle{Network Embedding} 
emerges as an effective way for learning node representations in recent years by leveraging different techniques, such as Markov random walks \cite{perozzi2014deepwalk}, matrix factorization \cite{yang2015network}, autoencoders \cite{cao2016deep}, and generative adversarial nets \cite{dai2018adversarial}. The learned node representations can boost various downstream learning tasks, such as node clustering \cite{wang2016structural} and node classification \cite{perozzi2014deepwalk}.

\paragraphtitle{Graph-based Semi-supervised Learning.}\quad
Early works in this line explicitly or implicitly adopt the assumption that nearby vertices are likely to have the same label. The most popular framework is label propagation \cite{Zhu03, Zhou03,chapelle2005semi,Wu12parw}, which formulates a quadratic regularization framework to implement the cluster assumption and enforce the consistency with labeled data. Other ideas include using graph partition techniques to place the cuts in low density regions \cite{blum2004semi}, using spectral kernels to learn smooth low-dimensional embeddings \cite{chapelle03,zhang06}, and modified adsorption \cite{talukdar2009new} and iterative classification algorithm \cite{sen2008collective}.
It was shown in \cite{ekambaram2013wavelet} and \cite{girault2014semi} that the graph regularization in many of these methods can be interpreted as low-pass graph filters. However, these methods are limited in their ability to incorporate vertex features for prediction. To jointly model graph structures and vertex features, many methods adopt an idea to regularize a supervised learner (e.g., support vector machines, neural networks) with some regularizer such as a Laplacian regularizer or an embedding-based regularizer. Examples include manifold regularization (LapSVM) \cite{belkin2006manifold}, deep semi-supervised embedding \cite{weston2012deep}, and Planetoid \cite{yang2016revisiting}.


\paragraphtitle{Graph Convolutional Neural Networks.}\quad
In the past few years, a series of works based on graph convolutional neural networks \cite{LiTBZ15,icml/QuBT19} have demonstrated promising performance in both semi-supervised and unsupervised learning. 

\textbf{Semi-supervised learning.} The pioneering work ChebyNet \cite{defferrard2016convolutional} exploits a $k$-th order polynomial filter via Chebyshev expansion to avoid the expensive eigen-decomposition in spectral graph convolution. GCN \cite{kipf2016semi} simplifies ChebyNet by designing an efficient layer-wise propagation rule with a first-order approximation, which inspires many follow-up works.
Specifically, graph attention networks  \cite{velickovic2017graph} and gated attention networks \cite{ZhangSXMKY18} introduce attention mechanisms to assign different weights to different nodes in a neighborhood. Some researches \cite{li2018deeper} provide deeper insights into GCN by showing the Laplacian smoothing nature and low-pass filtering nature of GCN-like models from spatial and spectral views, respectively. Instead of using two hops of neighbors, some later works such as JKNet \cite{xu2018representation}, LanczosNet \cite{liao2019lanczosnet}, MixHop \cite{icml/Abu-El-HaijaPKA19}, and GCNII \cite{icml/Chen20} propose to utilize multiple hops of neighbors.


\textbf{Unsupervised Learning.} One kind of methods stack several graph convolutional layers and learn node representations based on feature/structure preserving, e.g., graph autoencoder and graph variational autoencoder \cite{kipf2016variational}, GraphSAGE \cite{hamilton2017inductive}, marginalized graph autoencoder \cite{wang2017mgae}. Since these methods lack additional constraints to enhance the robustness of representations, adversarially regularized (variational) graph autoencoder \cite{pan2018adversarially} further adopts GAN to match the learned embeddings with a Gaussian prior. Another type of methods exploit contrastive learning to train an encoder to be contrastive between similar nodes and dissimilar ones for node representation learning, e.g., deep graph infomax \cite{velivckovic2018deep} and graphical mutual information \cite{www/PengHLZRXH20}. 

However, existing GCN-based methods mainly concentrate on 1-D graph convolution. To our knowledge, the only exception is \cite{monti2017geometric}, which explores 2-D graph convolution for matrix completion for recommendation. In this work, we investigate 2-D graph convolution for unsupervised and semi-supervised learning. 


\section{2-D Graph Convolution}

Graph signal processing is an active research field in recent years, which generalizes basic concepts in harmonic analysis, including signals, filters, Fourier transform, and convolution, to the graph domain. Given a graph $\mathcal{G}$ with a vertex set $\mathcal{V}$ and an adjacency matrix $\bm A$, if we associate each vertex $\nu_i$ with a real value $x_i$, then a signal on $\mathcal{G}$ can be defined as a vector $\bm x=[x_i]$. Graph filters are defined as mappings between input and output signals. 
Let $\bm G$ be a polynomial of $\bm A$ (usually normalized), i.e, $\bm G=p(\bm A)$, then $\bm G$ is a legit convolutional filter on graph $\mathcal{G}$, and the corresponding 1-D graph convolution is defined as
\begin{equation}
    \bm z = \bm G \bm x,
\end{equation}
where $\bm z$ is the output signal. Existing graph convolutional filters are all defined in this way and have achieved considerable success in various learning tasks on graphs.


However, previous research mainly focuses on the design and application of 1-D graph convolution. In this section, we present 2-D graph convolution for learning node representations. A comprehensive introduction to multi-dimensional graph convolution is provided by \cite{kurokawa2017multi}. Based on the theory developed by \cite{kurokawa2017multi}, we propose a localized 2-D graph convolution to circumvent the computationally intensive graph Fourier transform. Furthermore, we propose an even simpler dimensionwise separable 2-D graph convolution to efficiently model both object links and attribute relations.

\subsection{2-D Graph Signal and Spectral Convolution}
\paragraphtitle{2-D Graph Signal.}\quad A 2-D graph signal is a function defined on the Cartesian product of the vertex sets of two graphs. Formally, given two graphs ${\mathcal{G}}^{(1)}$ and ${\mathcal{G}}^{(2)}$ with $n$ and $m$ nodes respectively, and denote the vertex sets by ${\mathcal{\bm V}}^{(1)}$ and ${\mathcal{\bm V}}^{(2)}$. A real-valued signal defined on them is a function $x:{\mathcal{\bm V}}^{(1)}\times{\mathcal{\bm V}}^{(2)}\to\mathbb{R}$. For convenience, we simply denote $x(\nu^{(1)}_i, \nu^{(2)}_j)$ by $x_{ij}$ and organize then as a matrix:
\begin{equation}
    \bm X = (x_{ij})
    \in \mathbb{\bm R}^{n\times m},
    \quad x_{ij} = x(\nu^{(1)}_i, \nu^{(2)}_j),
\end{equation}
Signal $\bm X$ associates each node pairs $(\nu_i,\nu_j)\in\mathcal{V}^{(1)}\times \mathcal{V}^{(2)}$ with a real number $x_{ij}$. For example, the feature matrix given by usual node classification tasks is a 2-D signal defined on object link graph and attribute affinity graph.


\paragraphtitle{2-D Graph Fourier Transform}\quad
Define the graph Laplacian of ${\mathcal{G}}^{(1)}$ and ${\mathcal{G}}^{(2)}$ as $\bm L_{(1)}=\bm D^{(1)}-\bm A^{(1)}$ and $\bm L_{(2)}=\bm D^{(2)}-\bm A^{(2)}$, where $\bm A^{(1)}$, $\bm A^{(2)}$ are adjacency matrices and $\bm D^{(1)}$, $\bm D^{(2)}$ are the corresponding degree matrices.
\footnote{Discussion below also applies to row-normalized Laplacian $\bm L_r=\bm I - \bm D^{-1}\bm A$, column-normalized Laplacian, $\bm L_c = \bm I - \bm A \bm D^{-1}$, and symmetric normalized Laplacian $\bm L_s = \bm I - \bm D^{-1/2}\bm A \bm D^{-1/2}$.}
Assuming two Laplacian matrices have following eigen-decomposition
\begin{equation}
    \bm L_{(1)} = \bm U \bm \Lambda \bm U^{-1}, \quad
    \bm L_{(2)} = \bm V \bm M \bm V^{-1}, 
\end{equation}
where $\bm \Lambda, \bm M$ stores the eigenvalues $\lambda_i,\mu_j$  of two Laplacian matrices in their main diagonals, $\bm U=[\bm u_1,\cdots, \bm u_{n}]$ and $\bm V=[\bm v_1, \cdots, \bm v_{m}]$ store the corresponding unit-length eigen-vectors in their columns. This eigen-decomposition is always attainable for undirected graphs and nearly always attainable for directed graphs.

All $n\times m$ outer products $\bm u_i \bm v_j^\top$ together form a basis for the linear space $\mathbb{R}^{n\times m}$. It is known as 2-D graph Fourier basis -- an analogy of the Fourier basis in classical harmonic analysis in graph domain. 
A 2-D graph signal $\bm X$ can be decomposed in this basis with coefficients $s_{ij}$:
\begin{equation}
    \bm X=\sum\nolimits_{ij} s_{ij} \left(\bm u_i \bm v_j^\top\right),
\end{equation}
or in matrix form:
\begin{equation}\label{eq:inverse_fourier}
    \bm X = \bm U\bm S\bm V^\top,\; \text{where}\;\bm S= (s_{ij}) \in \mathbb{R}^{n\times m}.
\end{equation}
$\bm S$ is called the spectrum or Fourier coefficients of signal $\bm X$ and can be obtained by formula 
\begin{equation}\label{eq:fourier}
\bm S=\bm U^{-1}\bm X (\bm V^{-1})^\top.
\end{equation}
Eq.~(\ref{eq:fourier}) is so-called 2-D graph Fourier transform; Eq.~(\ref{eq:inverse_fourier}) is the corresponding inverse transform.

\paragraphtitle{2-D Spectral Graph Convolution}\quad
Given above decomposition, now we can manipulate the spectrum of 2-D signals and define 2-D spectral graph convolution.
Convolution is an operation that takes a signal as input and outputs another signal. By convolution theorem, convolution is equivalent to scaling entries of the spectrum. Thus, given a signal $X$ with spectrum $S$, a 2-D spectral graph convolution with $X$ as input is defined as:
\begin{equation}\label{eq:spectral_conv}
\bm Z = \bm U(\bm S\circ \bm P)\bm V^\top,
\end{equation}
where $\bm P$ is the spectral kernel (parameters in spectral domain) of this convolution, and `$\circ$' is Hadamard (entry-wise) product. 

\subsection{Fast Localized 2-D Graph Convolution}\label{sec:spatial_conv}


Although Eq.~(\ref{eq:spectral_conv}) well defines 2-D graph convolution, it is often impractical to perform convolution in the spectral domain, due to the high cost of computing the eigenbasis $\bm U$, $\bm V$ needed for Fourier transform. Similar to what \cite{defferrard2016convolutional} did to 1-D graph convolution, we propose 2-D spatial graph convolution here to avoid intensive computation. To achieve this goal, we need to parameterize the entries of spectral kernel $\bm P$ as a two-variable polynomial $p(\cdot,\cdot)$ of eigenvalues $\lambda, \mu$ such that $ P_{ij}=p(\lambda_i, \mu_j)$. Denote coefficients of the polynomial by $\bm \Theta = (\theta_{k_1k_2}) \in\mathbb{R}^{n\times m}$, then $\bm P$ is parameterized as
\begin{equation}
    P_{ij} = p(\lambda_i, \mu_j)= \sum_{k_1=0}^{n-1} \sum_{k_2=0}^{m-1} \theta_{k_1k_2} \lambda^{k_1}_i \mu^{k_2}_j.
\end{equation}
or in matrix form:
\begin{equation}\label{eq:parameterization}
    \bm P = \sum_{k_1=0}^{n-1} \sum_{k_2=0}^{m-1} \theta_{k_1k_2} \bm \Lambda^{k_1} \mathds{1} \bm M^{k_2},
\end{equation}
where $\mathds{1}$ denotes all-one matrix of shape $n\times m$. Substitute Eq.~(\ref{eq:parameterization}) into Eq.~(\ref{eq:spectral_conv}) and rearrange it, 2-D graph convolution will become
\begin{align}
    \bm Z 
    & = \sum\nolimits_{k_1k_2} \theta_{k_1k_2} 
    \bm U \left(\bm S\circ (\bm \Lambda^{k_1} \mathds{1} \bm M^{k_2})\right)\bm V^\top \nonumber\\
    & = \sum\nolimits_{k_1k_2} \theta_{k_1k_2} 
    \bm U \bm \Lambda^{k_1} \left(\bm S \circ \mathds{1} \right) \bm M^{k_2} \bm V^\top \nonumber \\
    & = \sum\nolimits_{k_1k_2} \theta_{k_1k_2} 
    \left( \bm U \bm \Lambda^{k_1} \bm U^{-1} \right) 
    \bm X 
    \left( \bm V \bm M^{k_2} \bm V^{-1} \right)^\top \nonumber \\
    & = \sum\nolimits_{k_1k_2} \theta_{k_1k_2} 
    \bm L_{(1)}^{k_1} \bm X (\bm L_{(2)}^{k_2})^\top
    \label{eq:spatial_conv} 
\end{align}
Eq.~(\ref{eq:spatial_conv}) is called 2-D spatial graph convolution, as it manipulates the signal $\bm X$ in the spatial domain. Eigen-decomposition of Laplacian and Fourier transformation of $\bm X$ is no longer required. Parameters $\bm \Theta = (\theta_{k_1k_2})$ in Eq.~(\ref{eq:spatial_conv}) are called the (spatial) kernel of this convolution. Similar to \cite{defferrard2016convolutional}, this spatial convolutional filter is localized. Denote by $K_1$ and $K_2$ the largest exponent of $\bm L_{(1)}$ and $\bm L_{(2)}$ in the polynomial, i.e., $k_1>K_1$ or $k_2>K_2$ implies $\theta_{k_1k_2}=0$.
The convoluted signal $z_{ij}$ of vertex pair $(\nu^{(1)}_i, \nu^{(2)}_j)$ only depends on the neighbourhood of $\nu^{(1)}_i$ within $K_1$ hops and the neighbourhood of $\nu^{(2)}_j$ within $K_2$ hops, so the filter is said to be $K_1$-localized on ${\mathcal{G}}^{(1)}$ and $K_2$-localized on ${\mathcal{G}}^{(2)}$, and the kernel $\bm \Theta$ is said to be of size $(K_1+1) \times (K_2+1)$.

\subsection{Dimensionwise Separable 2-D Graph Convolution (DSGC)}

Although the above spatial graph convolution avoids the computationally expensive Fourier transform, its general form with kernel size $K_1\times K_2$ still involves at least $K_1 \times K_2$ matrix multiplications. Inspired by the depth-wise separable convolution proposed in \cite{howard2017mobilenets}, we streamline spatial graph convolution by restricting the rank of $\bm \Theta$ to be one.
Consequently, $\bm \Theta$ is able to be decomposed as an outer product of two vectors $\bm \theta^{(1)}\in \mathbb{R}^n$ and $\bm \theta^{(2)}\in \mathbb{R}^m$, i.e., $\bm \Theta = \bm\theta^{(1)}\bm\theta^{(2)\top}$ and $\theta_{k_1k_2}=\theta^{(1)}_{k_1}\theta^{(2)}_{k_2}$. 
Finally, the 2-D spatial graph convolution in Eq.~(\ref{eq:spatial_conv}) becomes
\begin{align}
    \bm Z 
    & = \sum_{k_1=0}^{K_1} \sum_{k_2=0}^{K_2} \theta^{(1)}_{k_1}\theta^{(2)}_{k_2} 
    \bm L_{(1)}^{k_1} \bm X (\bm L_{(2)}^{k_2})^\top \\
    & = \left( \sum_{k_1=0}^{K_1}\theta^{(1)}_{k_1}\bm L_{(1)}^{k_1} \right) \bm X 
    \left( \sum_{k_2=0}^{K_2} \theta^{(2)}_{k_2} \bm L_{(2)}^{k_2} \right)^\top \label{eq:separable}
    =\bm G\bm X \bm F, \\
    &\text{where } \bm G=\sum_{k_1=0}^{K_1}\theta^{(1)}_{k_1}\bm L_{(1)}^{k_1}
    \text{ and } \bm F=\sum_{k_2=0}^{K_2} \theta^{(2)}_{k_2} (\bm L_{(2)}^{k_2})^\top.
\end{align}
We refer to Eq.~(\ref{eq:separable}) as dimensionwise separable graph convolution (DSGC), $\bm G$ as the \emph{object graph convolutional filter}, and $\bm F$ as the \emph{attribute graph convolutional filter}. The fastest way to compute it only requires $K_1+K_2$ matrix multiplications, much less than the $K_1 \times K_2$ matrix multiplications needed by a general 2-D spatial graph convolution. 
Fig.~\ref{fig:GXF} illustrates how $\bm G$ and $\bm F$ can work together to learn better node representations with DSGC.


\begin{figure}
    \centering
    \includegraphics[width=0.47\textwidth]{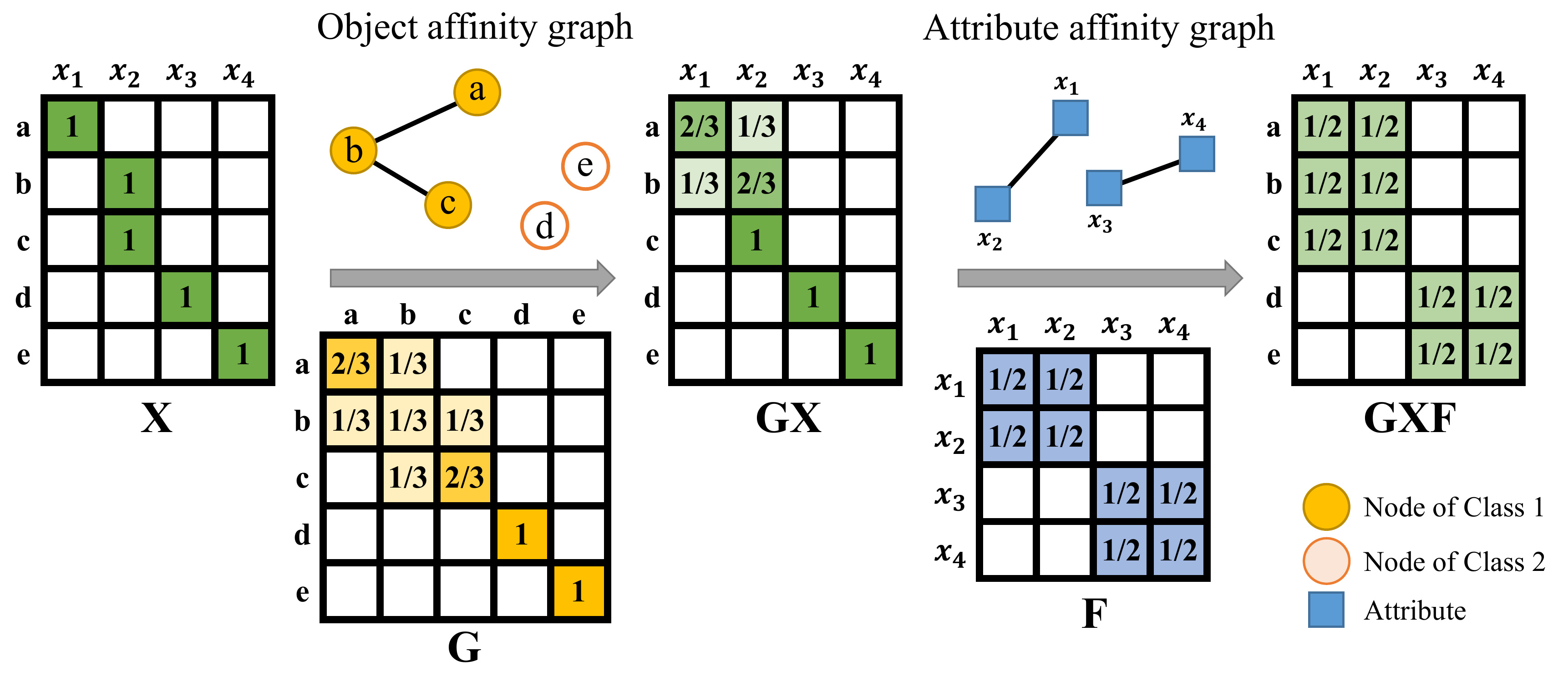}
    \vskip -0.1in
    \caption{Conceptual illustration of DSGC by a toy example. The node representations (row vectors) obtained by DSGC ($\bm G\bm X\bm F$) is better than those in the originals feature matrix ($\bm X$) or filtered by 1-D graph convolution ($\bm {GX}$), in the sense that nodes of the same class have more similar features.}\label{fig:GXF}
    \vskip -0.15in
\end{figure}
\section{Variance Reduction by DSGC}\label{sec:analysis}

Given a data distribution, the lowest possible error rate an classifier can achieve is the Bayes error rate \cite{fukunaga2013introduction}, which is caused by the intrinsic overlap between different classes and cannot be avoided. In this section, we show that DSGC with proper filters can reduce intra-class variance of the data distribution while keeping class centers roughly unchanged, hence reducing the overlap between classes and improving learning performance.

\paragraphtitle{Intra-class Variance and Inter-class Variance}\quad
Suppose samples $\x_i$ and their labels $y_i$ are observations of a random vector $\X=[\X_1, \cdots, \X_m]^\top$ and a random variable $\Y$ respectively. We define the variance of random vector $\X$ to be the sum of the variance of each dimension $\X_j$, i.e., the trace of the covariance matrix of $\X$. According to law of total variance \cite{grinstead2012introduction}, the variance of $\X$ can be divided into intra-class variance and inter-class variance:
\begin{align}
    \Var{\X}
    & = \underbrace{\E{\Var{\X|\Y}}}_{\text{Intra-class Variance}}
    + \underbrace{\Var{\E{\X|\Y}}}_{\text{Inter-class Variance}},
\end{align}
where the conditional variance $\Var{\X|\Y=k}$ is the variance of class $k$ and the conditional expectation $\E{\X|\Y=k}$ is the $k$-th class center. Intra-class variance (IntraVar) measures the average divergence within each class, while inter-class variance (InterVar) measures the divergence among class centers. We are interested in the IntraVar$/$InterVar ratio. Desired node representations should have low intra-class variance (i.e., compact and dense for each class), and high inter-classes variance (large margin between classes).

\subsection{Intra-class Variance Reduction by Object Graph Convolution}\label{sec:nodegraph}
Commonly used object graph convolution reduces variance by averaging over neighborhood. For any node $\nu_i$, \emph{object graph convolution} $\bm G\bm X$ produces a new feature vector $\z_i = \sum_j G_{ij} \x_j$. When $\bm G$ is a stochastic matrix, the output feature vector $\z_i$ is a weighted average of the neighbours of $\bm x_i$. Denote by $\Z$ a random vector of $\z_i$. Intuitively, as long as each node $i$ has enough same-class neighbours, $\Z$ will have a smaller $\IntraVar/\InterVar$ ratio than $\X$. 

Formally, assume that objects from different classes are connected with probability $q$, and classes are balanced, i.e., $\Pr(\Y=k)=1/K$ for each class $k$. Then, with the stochastic graph filter $\bm G={\bm D^{-1} \bm A^{(1)}}$, we have the following theorem.
\begin{theorem}
When $q$ is sufficiently small, the $\IntraVar/\InterVar$ ratio of $\Z$ is less than or equal to that of $\X$, i.e.,
	\begin{align}
	\frac{\E{\Var{\Z|\Y}}}{\Var{\E{\Z|\Y}}} \le \frac{\E{\Var{\X|\Y}}}{\Var{\E{\X|\Y}}}.
	\end{align}
\end{theorem}

The proof is given in the Appendices \ref{ap:theorem1}. This theorem tells that under the assumption that connected nodes are most likely to be of the same class, object graph convolution $\bm G\bm X$ can efficiently reduce the IntraVar/InterVar ratio.

\subsection{Intra-class Variance Reduction by Attribute Graph Convolution}\label{sec:attributegraph}
A proper attribute graph convolutional filter $\bm F$ can also reduce the IntraVar/InterVar ratio. We use the convention that the random vector $\X$ is a column vector, and hence the \emph{attribute graph convolution} $\bm X \bm F$ results in a new random vector $\bm F^\top\X$. We also assume that the node features are mean-centered, i.e. $\E{\X}=\bm0$.
\begin{theorem}\label{thm:F-intra}
    If the attribute graph convolutional filter $\bm F$ is a doubly stochastic matrix, then the output of attribute graph convolution has an intra-class variance less than or equal to that of \;$\X$, i.e.,
\begin{align*}
\sum\nolimits_i F_{ij} = \sum\nolimits_j &F_{ij} = 1 \;\text{and}\; F_{ij}\ge0,\forall\;i,j \\
    \qquad&\Rightarrow\quad \E{\Var{\bm F^\top\X|\Y}} \le \E{\Var{\X|\Y}}.
\end{align*}
\end{theorem}

The proof for Theorem~\ref{thm:F-intra} is given in the Appendices~\ref{ap:theorem23}.

To achieve a low IntraVar/InterVar ratio, in addition to reducing intra-class variance, we also need to keep the class centers apart after convolution, which then depends on the quality of the attribute affinity graph. A good attribute affinity graph should connect attributes that share similar expectations conditioned on $\Y$. Formally, each attribute $\X_j$ has $K$ conditional expectations w.r.t. $\Y$, which are denoted as a vector $\e_j=(\E{\X_j|\Y=1}, \cdots, \E{\X_j|\Y=K})\in\mathbb{R}^{K}$. We have the following.
\begin{theorem}\label{thm:F-inter}
    If \;$\forall F_{ij}\not=0$, $\norm{\e_i-\e_j}_2\le\varepsilon$, then the distance between $\e_j$ and $\widehat{\e}_j = \sum_i F_{ij}\e_i$ is also less than or equal to $\varepsilon$, i.e.,
    {\normalfont
    \begin{equation*}
        \norm{\e_i - \e_j}_2\le \varepsilon,\;\forall F_{ij}\not=0
        \quad\Rightarrow\quad \norm{\e_j - \widehat{\e}_j}_2 \le \varepsilon,
    \end{equation*}}
    and $\varepsilon$ can be arbitrarily small with a proper $\bm F$.
\end{theorem}

The proof for Theorem~\ref{thm:F-inter} is also given in the Appendices~\ref{ap:theorem23}. By Theorem \ref{thm:F-inter}, the conditional expectations of each attribute (i.e., class means) may change little after attribute graph convolution, and so does the inter-class variance. Combining Theorems \ref{thm:F-intra} \& \ref{thm:F-inter}, it suggests that a proper attribute affinity graph should connect attributes that have similar class means
, so as to achieve a low $\IntraVar/\InterVar$ ratio and improve performance.

\section{Unsupervised and Semi-Supervised Learning with DSGC}

\subsection{Learning Frameworks} \label{sec:frame}
Given an attributed graph with node feature matrix $\bm X$, we can learn node representations $\bm Z$ in an unsupervised manner by applying DSGC on $\bm X$, i.e.,
\begin{equation}\label{eq:separable-2}
    \bm Z=\bm G\bm X \bm F,
\end{equation}
and then perform various downstream learning tasks with the node representations $\bm Z$.

\paragraphtitle{Unsupervised Node Clustering}\quad
Any standard clustering algorithm can be applied on $\bm Z$ for clustering, as long as it is suitable for present data. In experiments, we use the popular spectral clustering method \cite{perona1998,von2007tutorial} along with linear kernel $\bm K=\bm Z\bm Z^\top$.

\paragraphtitle{Semi-supervised Node Classification}\quad
After obtaining the unsupervised node representations $\bm Z$, we may adopt any proper supervised classifier and train it with $\bm Z$ and a small portion of labels for semi-supervised classification. This two-step framework is semi-supervised in nature. In experiments, we choose a multi-layer perceptron with a single hidden layer as our classifier. In addition to the two-step framework, we can also plug DSGC into existing end-to-end graph-convolution-based methods. In experiments, we improve several popular methods, including GCN, GAT and GraphSAGE, by replacing their 1-D graph convolution with DSGC. For example, to incorporate DSGC into the vanilla GCN, we can modify the first layer propagation of GCN as:
\begin{equation}\label{GCN-F}
    \bm H^{(1)}=\sigma(\bm G\bm X\bm F\bm W^{(1)}),
\end{equation}
where $\bm H^{(1)}$ is the hidden units in the first layer, $\bm W^{(1)}\in \mathbb{R}^{m\times l}$ is the trainable parameters of GCN, and $\sigma$ is an activation function such as ReLU. 

Importantly, Eq.~(\ref{GCN-F}) can be considered as feeding a filtered feature matrix $\bm X\bm F$ instead of the raw feature matrix $\bm X$ to GCN. By our above analysis, a proper attribute graph convolutional filter $\bm F$ can reduce intra-class variance, which makes $\bm X\bm F$ much easier to classify and guarantees to help train a better model.
Furthermore, parameters of GCN is freely chosen from the parameter space $\bm  W^{(1)}$, while the model trained by Eq.~(\ref{GCN-F}) is restricted in a subspace $\bm F\bm W^{(1)}$. Since the chosen filter $\bm F$ is low-pass (Section~\ref{sec:filters}) and thus is nearly singular, $\bm F\bm W^{(1)}$ is a subspace of $\mathbb{R}^{m\times l}$ projected by $\bm F$. Model parameters in this subspace are generally better in terms of the generalization performance, due to the variance reduction property of $\bm F$. However, the model learned by Eq.~(\ref{GCN-F}) can hardly be learned by GCN, since the subspace $\bm F\bm W^{(1)}$ has measure zero, which is a tiny subset of $\mathbb{R}^{m\times l}$.

\begin{table}
	\centering
	\small
	\caption{Dataset statistics.}\label{tab:dataset_full}
	\vskip -0.1in
	\begin{tabular}{lrrrrc}
		\toprule
		Dataset & \#vertices & \#edges & \#cls & \#features & \makecell{ratio of intra-\\class edges}\\
		\midrule
		20 NG  & 18,846 & 147,034 & 20 & 11,697 &  96.8\% \\
		Wiki   & 3,767  & 129,597 & 9  & 18,316 &  38.0\% \\
		L-Cora & 11,881 & 64,898  & 10 & 3,780  &  76.5\% \\
		Corn. & 247 & 384 & 5 & 3371 & 23.7\% \\
		Texa. & 255 & 205 & 5 & 3371 & 19.8\% \\
		Wisc. & 320 & 721 & 5 & 3371 & 26.3\% \\
		Wash. & 265 & 417 & 5 & 3371 & 40.3\% \\
		\bottomrule
	\end{tabular}
\vskip -0.1in
\end{table}
\subsection{Implementation of Filters}\label{sec:filters}
\paragraphtitle{Object Graph Convolutional Filter}\quad
In most cases, the object graph $\bm A^{(1)}$ is given as part of the dataset, all we need is to design the filter. There are various graph convolutional filters available \cite{li2019label,kipf2016semi,defferrard2016convolutional,ijcai/0003LLW19}, but the key principle of filter design for semi-supervised and unsupervised learning is low-pass \cite{li2019label,nt2019revisiting}. Following this principle, we use the 2-order row-normalized affinity matrix as the object graph filter, i.e,
\begin{equation}
    \bm G = (\bm I- \bm L_{r}^{(1)})^2=(\bm D_1^{-1}\bm A^{(1)})^2.
\end{equation}

\paragraphtitle{Attribute Graph Convolutional Filter}\quad
A key issue in implementing DSGC is to construct a suitable attribute affinity graph ($\bm A^{(2)}$). Possible ways to construct attribute affinity graphs include extracting entity relation information from existing knowledge bases, building a similarity graph from features, or identifying correlations by domain knowledge. 
In experiments, we evaluate our methods on text dataset as in \cite{kipf2016semi,velickovic2017graph,hamilton2017inductive}, 
and leverage two suitable attribute affinity graphs for text data described below.

\emph{Positive point-wise mutual information} (PPMI) is a common tool for measuring the association between two words in computational linguistics \cite{church1990word}. PPMI between words $w_i$ and $w_j$ is defined by
    $\text{PPMI}(w_i,w_j)=\left[\log\frac{\Pr(w_i,w_j)}{\Pr(w_i)\Pr(w_j)}\right]_+$,
where $\Pr(w_i)$ is the probability of occurrence of word $w_i$, and $\Pr(w_i,w_j)$ is the probability of two words occurring together. If there is a semantic relation between two words, they usually tend to co-occur more frequently, and thus share a high PPMI value. 
Here, we use PPMI between words as the corresponding weights in the attribute affinity matrix $A^{(2)}$, and symmetrically normalize it as \cite{kipf2016semi}.

\emph{Word embedding based $k$-NN graphs.} Word embedding is a collection of techniques that map vocabularies to vectors in an Euclidean space. Embeddings of words are pre-trained vectors learned from corpus with algorithms such as GloVe \cite{pennington2014glove}. Since word embeddings capture semantic relations between words \cite{corr/abs-1801-09536}, they can be used for constructing an attribute affinity graph. With the embedding vectors, we can construct a $k$-NN graph with some proximity metric such as the Euclidean distance.

With the constructed attribute affinity graph ($\bm A^{(2)}$), we use one-step lazy random walk filter \cite{li2019label} in our experiments, i.e.,
\begin{equation}\label{F-emb}
    \bm F=(\bm I - \frac12 \bm L_s)=\frac12(\bm I+ \bm D_2^{-1/2}\bm A^{(2)}\bm D_2^{-1/2}).
\end{equation}

\section{Empirical Study}\label{sec:experiment}

We conduct extensive experiments in semi-supervised node classification and node clustering on seven real-world networks including 20 Newsgroups (20 NG) \cite{Lang95}, Large Cora (L-Cora) \cite{mccallumzy1999building,li2019label}, Wikispeedia (Wiki) \cite{west2009wikispeedia,west2012human}, and four subsets of WebKB (Cornell, Texas, Wisconsin, and Washington). 
\footnote{Note that we did not use the ``Cora'', ``Citeseer'' and ``PubMed'' datasets as in \cite{kipf2016semi,yang2016revisiting,sen2008collective}, since the attribute (word) lists are not provided.} Due to space limitation, on the four subsets of WebKB, we only report the results of semi-supervised node classification with some representative baselines, but the results are indicative enough.


\newcommand\stdfont{\footnotesize}

\begin{figure}[t]
   \centering
    \includegraphics[width=0.4\textwidth]{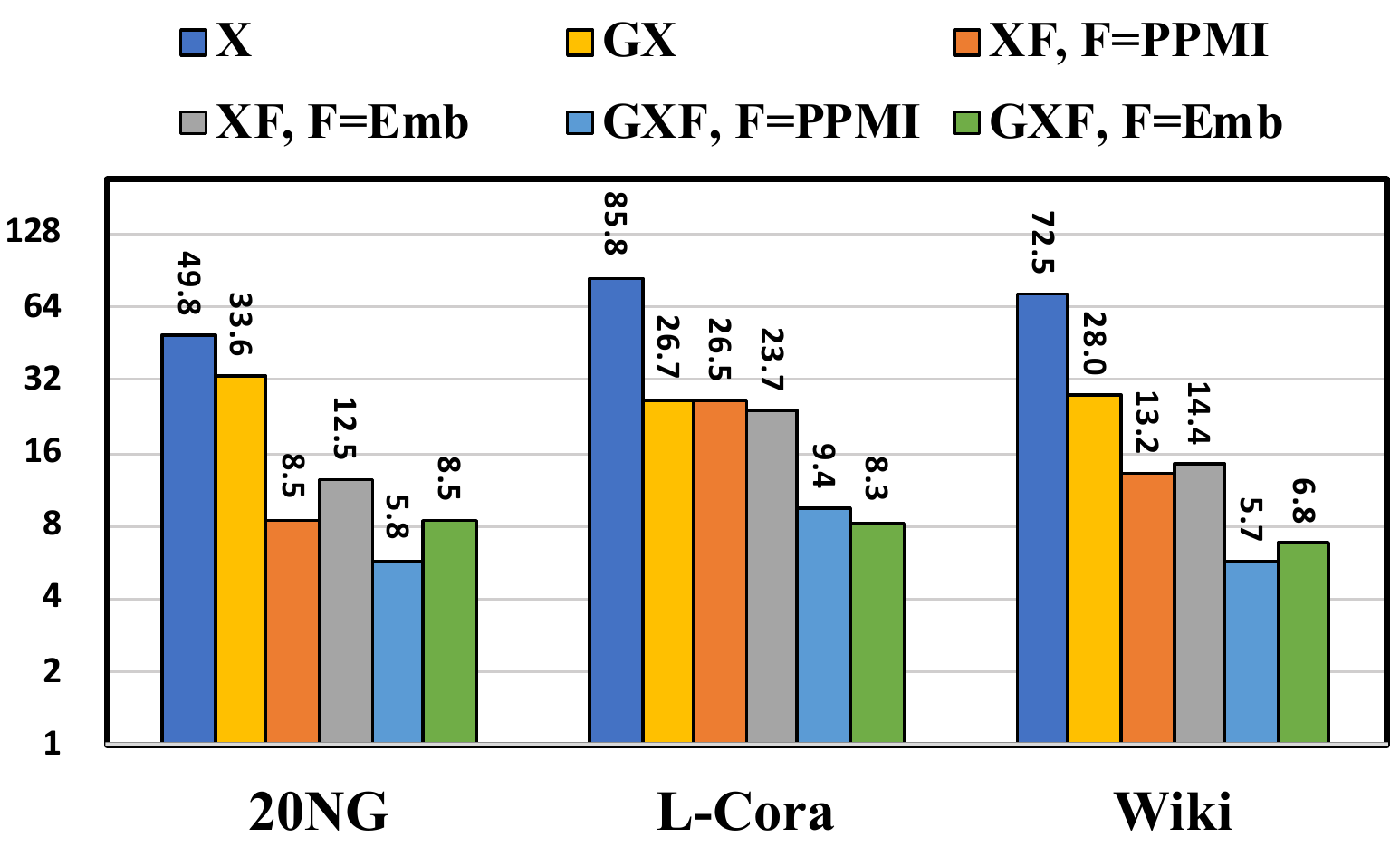}
    \vskip -0.1in
    \captionof{figure}{IntraVar/InterVar ratios. The lower, the better.}\label{fig:ratios}
    \vskip -0.1in
\end{figure}
\begin{table*}
    \centering
    \small
    \captionof{table}{Classification accuracy on 20 NG, L-Cora, and Wiki.}\label{tab:acc}
    \vskip -0.1in
    \begin{tabular}{lcc|rrrrrr}
        \toprule
              \multicolumn{3}{l}{Dataset}                         & \multicolumn{2}{c}{20 NG} & \multicolumn{2}{c}{L-Cora} & \multicolumn{2}{c}{Wiki} \\ 
              \cmidrule(r){1-3} \cmidrule(lr){4-5} \cmidrule(lr){6-7} \cmidrule(lr){8-9}
              \multicolumn{1}{l}{Method}  & $\bm G$ & \multicolumn{1}{c}{$\bm F$} & \multicolumn{1}{c}{20 labels/cls.} & \multicolumn{1}{c}{5 labels/cls.} & \multicolumn{1}{c}{20 labels/cls.} & \multicolumn{1}{c}{5 labels/cls.} & \multicolumn{1}{c}{20 labels/cls.} & \multicolumn{1}{c}{5 labels/cls.} \\
              \midrule
              \multicolumn{1}{l|}{MLP}        & \xmark & \xmark &        {63.76} {\stdfont$\pm$ 0.17} &        {38.67} {\stdfont$\pm$ 0.38} &        {52.97} {\stdfont$\pm$ 0.41} &        {39.56} {\stdfont$\pm$ 0.85} &        {67.23} {\stdfont$\pm$ 0.25} &        {54.41} {\stdfont$\pm$ 0.66} \\
              \multicolumn{1}{l|}{LP~\cite{Zhu03}}         & \cmark & \xmark &        {16.39} {\stdfont$\pm$ 0.20} &        { 8.62} {\stdfont$\pm$ 0.20} &        {55.77} {\stdfont$\pm$ 0.97} &        {38.97} {\stdfont$\pm$ 3.15} &        { 9.53} {\stdfont$\pm$ 0.05} &        {10.54} {\stdfont$\pm$ 0.19} \\
              \multicolumn{1}{l|}{GLP~\cite{li2019label}}        & \cmark & \xmark &        {74.99} {\stdfont$\pm$ 0.11} &        {52.62} {\stdfont$\pm$ 0.45} &        {68.95} {\stdfont$\pm$ 0.29} &        {56.42} {\stdfont$\pm$ 0.85} &        {60.05} {\stdfont$\pm$ 0.11} &        {48.45} {\stdfont$\pm$ 0.54} \\
              \multicolumn{1}{l|}{GCN~\cite{kipf2016semi}}        & \cmark & \xmark &        {76.25} {\stdfont$\pm$ 0.11} &        {53.78} {\stdfont$\pm$ 0.49} &        {67.75} {\stdfont$\pm$ 0.33} &        {54.27} {\stdfont$\pm$ 0.82} &        {59.81} {\stdfont$\pm$ 0.30} &        {47.93} {\stdfont$\pm$ 0.56} \\
              \multicolumn{1}{l|}{GAT~\cite{velickovic2017graph}}        & \cmark & \xmark &        {76.33} {\stdfont$\pm$ 0.16} &        {56.02} {\stdfont$\pm$ 0.57} &        {68.88} {\stdfont$\pm$ 0.78} &        {56.89} {\stdfont$\pm$ 1.53} &        {50.97} {\stdfont$\pm$ 0.54} &        {46.99} {\stdfont$\pm$ 0.83} \\
              \multicolumn{1}{l|}{DGI~\cite{velivckovic2018deep}}        & \cmark & \xmark &        {73.34} {\stdfont$\pm$ 0.27} & \textbf{66.57} {\stdfont$\pm$ 0.63} &        {61.39} {\stdfont$\pm$ 0.50} &        {54.77} {\stdfont$\pm$ 1.24} &        {49.70} {\stdfont$\pm$ 1.63} &        {43.64} {\stdfont$\pm$ 1.89} \\
              \multicolumn{1}{l|}{GraphSAGE~\cite{hamilton2017inductive}}  & \cmark & \xmark &        {65.73} {\stdfont$\pm$ 0.17} &        {42.48} {\stdfont$\pm$ 0.77} &        {57.28} {\stdfont$\pm$ 0.71} &        {46.79} {\stdfont$\pm$ 1.91} &        {65.52} {\stdfont$\pm$ 0.62} &        {48.81} {\stdfont$\pm$ 0.76} \\
              \multicolumn{1}{l|}{GCNII~\cite{icml/Chen20}}  & \cmark & \xmark &        \textbf{77.41} {\stdfont$\pm$ 0.12} &        {58.10} {\stdfont$\pm$ 0.85} &        {68.18} {\stdfont$\pm$ 0.19} &        {57.02} {\stdfont$\pm$ 0.55} &        {43.65} {\stdfont$\pm$ 1.03} &        {35.98} {\stdfont$\pm$ 4.45} \\
              \multicolumn{1}{l|}{JK-MaxPool~\cite{xu2018representation}}  & \cmark & \xmark &        {71.00} {\stdfont$\pm$ 0.19} &        {49.09} {\stdfont$\pm$ 0.27} &        {67.44} {\stdfont$\pm$ 0.18} &        {51.63} {\stdfont$\pm$ 0.52} &        {45.26} {\stdfont$\pm$ 0.37} &        {44.13} {\stdfont$\pm$ 0.38} \\
              \multicolumn{1}{l|}{JK-Concat~\cite{xu2018representation}}  & \cmark & \xmark &        {72.24} {\stdfont$\pm$ 0.13} &        {49.76} {\stdfont$\pm$ 0.27} &        {67.47} {\stdfont$\pm$ 0.19} &        {51.96} {\stdfont$\pm$ 0.46} &        {47.24} {\stdfont$\pm$ 0.30} &        {45.21} {\stdfont$\pm$ 0.29} \\
              \multicolumn{1}{l|}{GRAND~\cite{feng2020graph}}  & \cmark & \xmark &        {74.45} {\stdfont$\pm$ 0.72} & {57.97} {\stdfont$\pm$2.79}& {69.30} {\stdfont$\pm$ 0.59} & {52.12} {\stdfont$\pm$ 0.73}& {62.25} {\stdfont$\pm$ 0.93} & {47.17} {\stdfont$\pm$ 3.38}\\
              \midrule
              \multicolumn{1}{l|}{DSGC ($\bm G\bm X$)}
              								  & \cmark & \xmark &        {75.60} {\stdfont$\pm$ 0.13} &        {53.84} {\stdfont$\pm$ 0.46} &        {67.74} {\stdfont$\pm$ 0.30} &        {55.67} {\stdfont$\pm$ 0.72} &        {58.73} {\stdfont$\pm$ 0.34} &        {47.34} {\stdfont$\pm$ 0.54} \\
              \cmidrule{1-9}
              \multicolumn{1}{l|}{\multirow{2}{*}{\makecell{DSGC ($\bm X\bm F$)}}}
                                              & \xmark & Emb    &        {66.27} {\stdfont$\pm$ 0.13} &        {48.04} {\stdfont$\pm$ 0.38} &        {58.70} {\stdfont$\pm$ 0.30} &        {46.41} {\stdfont$\pm$ 0.55} & \textbf{69.76} {\stdfont$\pm$ 0.20} & \textbf{59.76} {\stdfont$\pm$ 0.58} \\
              \multicolumn{1}{l|}{}           & \xmark & PPMI   &        {75.36} {\stdfont$\pm$ 0.11} &        {59.61} {\stdfont$\pm$ 0.34} &        {61.01} {\stdfont$\pm$ 0.23} &        {48.31} {\stdfont$\pm$ 0.62} & \textbf{69.91} {\stdfont$\pm$ 0.21} & \textbf{60.13} {\stdfont$\pm$ 0.61} \\
              \cmidrule{1-9}
              \multicolumn{1}{l|}{\multirow{2}{*}{\makecell{DSGC ($\bm G\bm X\bm F$)}}}
                                              & \cmark & Emb    & {76.53} {\stdfont$\pm$ 0.15} &        {59.91} {\stdfont$\pm$ 0.31} & \textbf{69.81} {\stdfont$\pm$ 0.26} & \textbf{58.63} {\stdfont$\pm$ 0.75} &        {60.50} {\stdfont$\pm$ 0.26} &        {49.69} {\stdfont$\pm$ 0.56} \\
              \multicolumn{1}{l|}{}           & \cmark & PPMI   & \textbf{81.69} {\stdfont$\pm$ 0.12} & \textbf{68.94} {\stdfont$\pm$ 0.32} & \textbf{70.20} {\stdfont$\pm$ 0.24} & \textbf{59.43} {\stdfont$\pm$ 0.68} &        {58.84} {\stdfont$\pm$ 0.26} &        {48.51} {\stdfont$\pm$ 0.54} \\
              \bottomrule
        \multicolumn{6}{l}{$^\star$ \cmark\, and \xmark\, indicate using/not using $\bm G$ or $\bm F$.}
    \end{tabular}
\end{table*}

\begin{table}[]
    \centering
    \small
    \setlength{\tabcolsep}{2pt}
    \caption{Classification accuracy on four subsets of WebKB.}\label{tab:webkb}
    \vskip -0.1in

\begin{tabular}{l|c|rrrr}
    \toprule
    \multicolumn{1}{l|}{Method} & $\bm F$ & \multicolumn{1}{c}{Corn.} & \multicolumn{1}{c}{Texa.} & \multicolumn{1}{c}{Wisc.} & \multicolumn{1}{c}{Wash.} \\
    \midrule
    GCN~\cite{kipf2016semi}	& \xmark &  {50.25} {\stdfont$\pm$0.66}  & {60.31} {\stdfont$\pm$1.06}  & {52.74} {\stdfont$\pm$0.92}  & {53.83} {\stdfont$\pm$1.36}  \\
    GLP~\cite{li2019label}	& \xmark &  {50.86} {\stdfont$\pm$0.75}  & {59.40} {\stdfont$\pm$1.27}  & {55.28} {\stdfont$\pm$0.97}  & {55.87} {\stdfont$\pm$1.32}  \\
    GAT~\cite{velickovic2017graph}	& \xmark &  {51.21} {\stdfont$\pm$1.40}  & {60.06} {\stdfont$\pm$1.08}  & {52.92} {\stdfont$\pm$1.18}  & {56.85} {\stdfont$\pm$2.00}  \\
    GRAND~\cite{feng2020graph}	& \xmark &  {49.01} {\stdfont$\pm$1.04}  & {57.38} {\stdfont$\pm$0.26}  & {49.30} {\stdfont$\pm$0.22}  & {42.09} {\stdfont$\pm$1.25}  \\
    \midrule
    DSGC ($\bm G\bm X$)  & \xmark &  {51.22} {\stdfont$\pm$0.77}  & {59.35} {\stdfont$\pm$1.26}  & {56.26} {\stdfont$\pm$0.93}  & {55.77} {\stdfont$\pm$1.30}  \\
    \midrule
    \multirow{2}{*}{\makecell{DSGC ($\bm X \bm F$)}}	& Emb &  \textbf{62.14} {\stdfont$\pm$1.02}  & \textbf{68.00} {\stdfont$\pm$0.75}  & \textbf{73.50} {\stdfont$\pm$0.81}  & \textbf{65.78} {\stdfont$\pm$1.27}  \\
    & PPMI	&  \textbf{60.10} {\stdfont$\pm$0.91}  & \textbf{67.34} {\stdfont$\pm$0.94}  & \textbf{72.88} {\stdfont$\pm$0.78}  & \textbf{65.40} {\stdfont$\pm$1.35}  \\
    \midrule
    \multirow{2}{*}{\makecell{DSGC ($\bm G \bm X \bm F$)}}	& Emb &  {53.02} {\stdfont$\pm$0.67}  & {61.89} {\stdfont$\pm$0.94}  & {58.64} {\stdfont$\pm$1.12}  & {59.05} {\stdfont$\pm$1.21}  \\
    & PMI	&  {52.35} {\stdfont$\pm$0.52}  & {61.83} {\stdfont$\pm$0.91}  & {56.40} {\stdfont$\pm$1.04}  & {57.01} {\stdfont$\pm$1.20}  \\
\bottomrule
\end{tabular}

\vskip -0.1in
\end{table}

\subsection{Datasets}
The statistics of all datasets are summarized in Table~\ref{tab:dataset_full}, where the last row shows the intra-class edge ratio of the object link graph of each dataset, which can reflect the quality of the graph.

\paragraphtitle{20 Newsgroups} (20 NG) \cite{Lang95} is an email discussion group, where each object is an email and there are 18846 emails in total. Each email is represented by an 11697-dimension tf-idf feature vector. Two emails are connected by an edge if they replies the same one.

\paragraphtitle{Wikispeedia} (Wiki) \cite{west2009wikispeedia,west2012human} is a webpage network in which the objects are 3767 Wikipedia webpages, and the edges are web hyperlinks. Each webpage is described by a 18316-dimension tf-idf vector. We removed several tiny classes, so the webpages distribute more evenly across the remaining 9 categories.

\paragraphtitle{Large Cora} (L-Cora) \cite{mccallumzy1999building} is a citation network in which the objects are computer science research papers represented by 3780 dimension of tf-idf values. Two papers are connected by an undirected edge if and only if one cites the other. These citation links form a object graph.
After removing the papers that belong to no topic and the ones that have no authors or title, a subset of 11881 papers is obtained \cite{sacca2013collective}. 
We name this dataset ``Large Cora'' to distinguish it from the ``Cora'' dataset with 2708 papers used in \cite{kipf2016semi,yang2016revisiting,sen2008collective}.

\paragraphtitle{WebKB}
\footnote{http://www.cs.cmu.edu/afs/cs.cmu.edu/project/theo-20/www/data/}
is a webpage dataset collected by Carnegie Mellon University. We use its four subdatasets: Cornell, Texas, Wisconsin, and Washington. The objects are webpages, and the edges are hyperlinks between them. The webpages are represented by tf-idf feature vectors, and manually classified into five categories: student, project, course, staff, and faculty.

\subsection{Variance Reduction and Visualization}\quad
First of all, to verify our analysis in section~\ref{sec:analysis}, we illustrate the variance reduction effect of both object graph convolution and attribute graph convolution. As shown in Figure~\ref{fig:ratios}, 1-D graph convolution ($\bm {GX}$ or $\bm {XF}$) already greatly reduces the $\IntraVar/\InterVar$ ratio, and 2-D graph convolution ($\bm {GXF}$) reduces it even further.

In Figure~\ref{fig:visualization}, we visualize the results of performing graph convolution on the object features of 20 NG by t-SNE. It can be seen that both object graph convolution and attribute graph convolution can successfully reduce the overlap among classes, and 2-D graph convolution (DSGC) is more effective than 1-D.

\subsection{Semi-supervised Node Classification}

\begin{table}
	\centering
	\small
	\captionof{table}{Baselines improved by DSGC.}\label{tab:baselinewithf}
    \vskip -0.1in
	\begin{tabular}{l|c|lll}
		\toprule
		Methods & $\bm F$ & \multicolumn{1}{c}{20 NG} & \multicolumn{1}{c}{L-Cora} & \multicolumn{1}{c}{Wiki} \\
		\midrule
		\multirow{2}{*}{GAT~\cite{velickovic2017graph}}    & \xmark & 76.33 {\stdfont$\pm$ 0.16} & 68.88 {\stdfont$\pm$ 0.78} & 50.97 {\stdfont$\pm$ 0.54} \\
                        		& PPMI   & 78.01 {\stdfont$\pm$ 0.30} & 67.38 {\stdfont$\pm$ 0.65} & 55.43 {\stdfont$\pm$ 0.51} \\
		\cmidrule{1-5}
		\multirow{2}{*}{GCN~\cite{kipf2016semi}}    & \xmark & 76.25 {\stdfont$\pm$ 0.11} & 67.75 {\stdfont$\pm$ 0.33} & 59.81 {\stdfont$\pm$ 0.30} \\
                        		& PPMI   & 81.60 {\stdfont$\pm$ 0.10} & 67.87 {\stdfont$\pm$ 0.25} & 61.33 {\stdfont$\pm$ 0.28} \\
		\cmidrule{1-5}
		\multirow{2}{*}{GraphSAGE~\cite{hamilton2017inductive}}  & \xmark & 65.73 {\stdfont$\pm$ 0.17} & 57.28 {\stdfont$\pm$ 0.71} & 65.52 {\stdfont$\pm$ 0.62} \\
                        		& PPMI   & 76.27 {\stdfont$\pm$ 0.33} & 60.23 {\stdfont$\pm$ 1.81} & 67.26 {\stdfont$\pm$ 0.52} \\
		\bottomrule
		\multicolumn{5}{l}{$^\star$ \xmark\, indicates not using $\bm F$.}
	\end{tabular}
\vskip -0.15in
\end{table}

\paragraphtitle{Baselines}\quad
We compare DSGC with the following baselines: label propagation (LP) \cite{Wu12parw}, multi-layer perceptron (MLP), graph convolutional networks (GCN) \cite{kipf2016semi}, generalized label propagation (GLP) \cite{li2019label}, GraphSAGE \cite{hamilton2017inductive}, graph attention networks (GAT) \cite{velickovic2017graph}, deep graph infomax (DGI) \cite{velivckovic2018deep}, GCNII \cite{icml/Chen20}, jumping knowledge networks (JK-MaxPool \& JK-Concat) \cite{xu2018representation}, and GRAND\cite{feng2020graph}. We also try to improve GCN, GAT, and GraphSAGE by DSGC as described in Section \ref{sec:frame}. Our methods with mere object graph filter ($\bm G$) or mere attribute graph filter ($\bm F$) are also tested, for the purpose of ablation study. PPMI and Emb denote attribute affinity graphs constructed by positive point-wise mutual information and word embedding respectively as described in Section \ref{sec:filters}).

\begin{table*}
    \centering
    \caption{Clustering performance.}\label{tab:clustering}
    \vskip -0.1in
    \small
    \begin{tabular}{lcc|cccccccc}
        \toprule
        \multicolumn{3}{l}{Dataset} & \multicolumn{2}{c}{20 NG} & \multicolumn{2}{c}{L-Cora} & \multicolumn{2}{c}{Wiki}\\
        \cmidrule(lr){1-3}\cmidrule(lr){4-5} \cmidrule(lr){6-7} \cmidrule(lr){8-9}
        Method                  & $\bm G$ & \multicolumn{1}{c}{$\bm F$} & {Acc(\%)} & {NMI(\%)} & {Acc(\%)} & {NMI(\%)} & {Acc(\%)} & {NMI(\%)}\\
        \midrule
        \multicolumn{1}{l|}{Spectral} & \xmark  & \xmark & 25.29 {\footnotesize $\pm$ 1.01} & 28.18 {\footnotesize $\pm$ 0.74} & 28.22 {\footnotesize $\pm$ 1.01} & 11.61 {\footnotesize $\pm$ 0.04} & 29.25 {\footnotesize $\pm$ 0.00} & 21.83 {\footnotesize $\pm$ 0.00}\\
        \multicolumn{1}{l|}{GAE~\cite{kipf2016variational}}  & \cmark & \xmark& 38.92 {\footnotesize $\pm$ 1.39} & 44.58 {\footnotesize $\pm$ 0.40} & 34.45 {\footnotesize $\pm$ 0.76}& 22.38 {\footnotesize $\pm$ 0.18} & 33.78 {\footnotesize $\pm$ 0.32}& 22.88 {\footnotesize $\pm$ 0.20}\\
        \multicolumn{1}{l|}{VGAE~\cite{kipf2016variational}}  & \cmark & \xmark & 25.04 {\footnotesize $\pm$ 0.81} & 25.72 {\footnotesize $\pm$ 0.77} & 29.45 {\footnotesize $\pm$ 1.25}& 17.53 {\footnotesize $\pm$ 0.15} & 33.83 {\footnotesize $\pm$ 0.45} & 21.46 {\footnotesize $\pm$ 0.19}\\
        \multicolumn{1}{l|}{MGAE~\cite{wang2017mgae}} & \cmark & \xmark& 47.83 {\footnotesize $\pm$ 2.33}& \textbf{56.14} {\footnotesize $\pm$ 1.00} & 35.87 {\footnotesize $\pm$ 0.97}& 30.57 {\footnotesize $\pm$ 0.98} & 32.73 {\footnotesize $\pm$ 1.16}& 27.95 {\footnotesize $\pm$ 2.29}\\
        \multicolumn{1}{l|}{ARGE~\cite{pan2018adversarially}} & \cmark  & \xmark& 42.04 {\footnotesize $\pm$ 0.50}& 44.13 {\footnotesize $\pm$ 0.91} & 36.07 {\footnotesize $\pm$ 0.05} & 27.74 {\footnotesize $\pm$ 0.01} & 26.49 {\footnotesize $\pm$ 0.10} & 17.17 {\footnotesize $\pm$ 0.05}\\
        \multicolumn{1}{l|}{ARVGE~\cite{pan2018adversarially}} & \cmark  & \xmark & 21.10 {\footnotesize $\pm$ 0.61} & 21.79 {\footnotesize $\pm$ 0.49} & 26.45 {\footnotesize $\pm$ 0.03} & 12.94 {\footnotesize $\pm$ 0.01}& 33.82 {\footnotesize $\pm$ 0.13} & 21.42 {\footnotesize $\pm$ 0.11}\\
        \multicolumn{1}{l|}{AGC~\cite{ijcai/0003LLW19}} & \cmark  & \xmark & 38.83 {\footnotesize $\pm$ 0.84} & 47.08 {\footnotesize $\pm$ 1.57} & \textbf{41.76} {\footnotesize $\pm$ 0.01} & \textbf{33.65} {\footnotesize $\pm$ 0.01} & 32.74 {\footnotesize $\pm$ 0.01} & 24.90 {\footnotesize $\pm$ 0.01}\\
        \midrule
        \multicolumn{1}{l|}{DSGC ($\bm G \bm X$)}
        & \cmark  & \xmark   & 38.42 {\footnotesize $\pm$ 0.66} & 46.28 {\footnotesize $\pm$ 0.93}& 38.26 {\footnotesize $\pm$ 0.02} & 30.66 {\footnotesize $\pm$ 0.02} & 31.43 {\footnotesize $\pm$ 0.09} & 24.16 {\footnotesize $\pm$ 0.18}\\
        \cmidrule{1-9}
        \multicolumn{1}{l|}{\multirow{2}{*}{\makecell{DSGC ($\bm X \bm F$)}}}   & \xmark & Emb& 28.99 {\footnotesize $\pm$ 0.06}& 33.22 {\footnotesize $\pm$ 0.10} & 30.80 {\footnotesize $\pm$ 0.56}& 17.46 {\footnotesize $\pm$ 0.21} & \textbf{35.45} {\footnotesize $\pm$ 0.91} & \textbf{33.44} {\footnotesize $\pm$ 0.66}\\
        \multicolumn{1}{l|}{}   & \xmark   & PPMI   & \textbf{48.36} {\footnotesize $\pm$ 2.40} & 53.27 {\footnotesize $\pm$ 2.17} & 36.46 {\footnotesize $\pm$ 0.06}& 22.53 {\footnotesize $\pm$ 0.03} & \textbf{38.10} {\footnotesize $\pm$ 0.01} & \textbf{36.07} {\footnotesize $\pm$ 0.02}\\
        \cmidrule{1-9}
        \multicolumn{1}{l|}{\multirow{2}{*}{\makecell{DSGC ($\bm G \bm X \bm F$)}}}  & \cmark & Emb   & 43.40 {\footnotesize $\pm$ 0.66} & 50.97 {\footnotesize $\pm$ 0.58} & 40.75 {\footnotesize $\pm$ 0.02} & \textbf{33.05} {\footnotesize $\pm$ 0.04} & 30.50 {\footnotesize $\pm$ 0.01}& 25.48 {\footnotesize $\pm$ 0.03}\\
        \multicolumn{1}{l|}{}   & \cmark  & PPMI& \textbf{52.25} {\footnotesize $\pm$ 1.97} & \textbf{61.34} {\footnotesize $\pm$ 1.07}& \textbf{41.24} {\footnotesize $\pm$ 0.04} & 30.92 {\footnotesize $\pm$ 0.01} & 31.37 {\footnotesize $\pm$ 0.08}& 26.06 {\footnotesize $\pm$ 0.20}\\
        \bottomrule
        \multicolumn{9}{l}{$^\star$ \cmark\, and \xmark\, indicate using/not using $\bm G$ or $\bm F$.}
    \end{tabular}
    \vskip -0.1in
\end{table*}

\paragraphtitle{Settings}\quad
 For 20 NG, L-Cora and Wiki, we test two scenarios -- 20 labels/class and 5 labels/class. We follow GCN \cite{kipf2016semi} and many others to set aside a validation set containing 500 samples for hyper-parameter tuning. For the four small subdatasets of WebKB, we randomly split them into 5/30/65\% as train/valid/test set, and ensure each class has at least 1 label. Hyper-parameters of all methods, including ours and baselines, are tuned by grid search according to validation. The reported results of all methods are averaged over 50 runs. More experimental details are provided in Appendix \ref{ap:experiment}. 
 
\paragraphtitle{Performance}\quad
Classification accuracies are summarized in Tables~\ref{tab:acc} and ~\ref{tab:webkb}, and the top 2 accuracies are highlighted. Results of improved GAT, GCN, and GraphSAGE are shown in Table~\ref{tab:baselinewithf}. The following observations can be made.
Firstly, object graph convolution ($\bm G$) does not always help. On 20 NG and L-Cora, methods based on it like DSGC ($\bm {GX}$), GCN and GAT all outperform MLP significantly. However, on Wiki and the four subsets of WebKB, object graph convolution severely harms the performance. This is because the hyperlink graphs are highly noisy. Only a small portion of edges connect nodes of the same class, much lower than that of 20 NG (96.8\%) and L-Cora (76.5\%) (see Table~\ref{tab:dataset_full}). This shows the limitation of object graph convolution.
Secondly, attribute graph convolution works. As shown in Table~\ref{tab:acc}, DSGC with mere $\bm F$ already outperforms MLP significantly. Especially, on Wiki and WebKB, object graph convolution fails while attribute graph convolution is still effective.
Thirdly, 2-D graph convolution is useful. On datasets with good object link graphs like 20 NG and L-Cora, DSGC with both $\bm G$ and $\bm F$ performs much better than with either one of them only. 
Especially, DSGC ($\bm G \bm X \bm F$) with PPMI achieves the best performance among all methods.
On datasets with bad object link graphs such as Wiki and the four subsets of WebKB, DSGC with both $\bm G$ and $\bm F$ improves upon DSGC with mere $\bm G$ and outperforms most baselines. Especially, DSGC with mere $\bm F$ achieves the best performance, which improves upon the best baseline by 3.63\% and 5.72\% in absolute accuracy in two scenarios.

Remarkably, it can be seen from Table~\ref{tab:baselinewithf} that by incorporating DSGC, the performance of baselines including GCN, GAT, GraphSAGE is improved substantially in most cases, which again confirms that attribute graph convolution is a useful complement to object graph convolution.



\subsection{Node Clustering}

\paragraphtitle{Baselines}\quad
We test the proposed node clustering method with DSGC (Section~\ref{sec:frame}) with or without $\bm G$ and $\bm F$ in five cases, and compare them with existing strong baselines including GAE and VGAE \cite{kipf2016variational}, MGAE \cite{wang2017mgae}, ARGE and ARVGE \cite{pan2018adversarially} and AGC \cite{ijcai/0003LLW19}. We also compare with the spectral clustering (Spectral) method that operates on a similarity graph constructed by a linear kernel. Detailed experiment settings are included in Appendix \ref{ap:experiment}.

\paragraphtitle{Performance}\quad
We adopt two widely-used clustering measures \cite{Aggarwal14}: clustering accuracy (Acc) and normalized mutual information (NMI), and the results are shown in Table~\ref{tab:clustering} with the top 2 results highlighted. We can make the following observations. 1) Attribute graph convolution is highly effective. On 20 NG, DSGC ($\bm X \bm F$) with PPMI outperforms most baselines by a very large margin. On Wiki, DSGC ($\bm X \bm F$) with PPMI or Emb significantly outperforms all the baselines. 2) 2-D graph convolution is beneficial as validated in the classification experiments. On 20 NG, DSGC ($\bm {GXF}$) with PPMI can further improve upon the already very strong performance of DSGC ($\bm X \bm F$) with PPMI and performs the best; On L-Cora, DSGC ($\bm {GXF}$) with PPMI or Emb improves upon either DSGC ($\bm G \bm X$)  or DSGC ($\bm X \bm F$) and outperforms most baselines significantly. On Wiki, DSGC ($\bm X \bm F$) performs better than DSGC ($\bm {GXF}$), due to the low-quality object link graph as explained above.



\section{Conclusion}

We have proposed a simple and efficient dimensionwise separable 2-D graph convolution (DSGC) for unsupervised and semi-supervised learning on graphs. We have demonstrated theoretically and empirically that by exploiting attribute relations in addition to object relations, DSGC can learn better node representations than existing methods based on the regular 1-D graph convolution, leading to promising performance on node classification and clustering tasks. We believe DSGC can be applied to a wide variety of applications such as action recognition, malware detection, and recommender systems. In future work, we plan to apply DSGC to solve more practical problems.

\begin{acks}
We would like to thank the anonymous reviewers for their insightful comments. This research was supported by the General Research Fund No.15222220 funded by the UGC of Hong Kong.
\end{acks}


\bibliographystyle{ACM-Reference-Format}
\bibliography{parw}


\begin{thebibliography}{65}


\ifx \showCODEN    \undefined \def \showCODEN     #1{\unskip}     \fi
\ifx \showDOI      \undefined \def \showDOI       #1{#1}\fi
\ifx \showISBNx    \undefined \def \showISBNx     #1{\unskip}     \fi
\ifx \showISBNxiii \undefined \def \showISBNxiii  #1{\unskip}     \fi
\ifx \showISSN     \undefined \def \showISSN      #1{\unskip}     \fi
\ifx \showLCCN     \undefined \def \showLCCN      #1{\unskip}     \fi
\ifx \shownote     \undefined \def \shownote      #1{#1}          \fi
\ifx \showarticletitle \undefined \def \showarticletitle #1{#1}   \fi
\ifx \showURL      \undefined \def \showURL       {\relax}        \fi
\providecommand\bibfield[2]{#2}
\providecommand\bibinfo[2]{#2}
\providecommand\natexlab[1]{#1}
\providecommand\showeprint[2][]{arXiv:#2}

\bibitem[\protect\citeauthoryear{Abu{-}El{-}Haija, Perozzi, Kapoor,
  Alipourfard, Lerman, Harutyunyan, Steeg, and Galstyan}{Abu{-}El{-}Haija
  et~al\mbox{.}}{2019}]%
        {icml/Abu-El-HaijaPKA19}
\bibfield{author}{\bibinfo{person}{Sami Abu{-}El{-}Haija},
  \bibinfo{person}{Bryan Perozzi}, \bibinfo{person}{Amol Kapoor},
  \bibinfo{person}{Nazanin Alipourfard}, \bibinfo{person}{Kristina Lerman},
  \bibinfo{person}{Hrayr Harutyunyan}, \bibinfo{person}{Greg~Ver Steeg}, {and}
  \bibinfo{person}{Aram Galstyan}.} \bibinfo{year}{2019}\natexlab{}.
\newblock \showarticletitle{MixHop: Higher-Order Graph Convolutional
  Architectures via Sparsified Neighborhood Mixing}. In
  \bibinfo{booktitle}{\emph{{ICML}}}. \bibinfo{pages}{21--29}.
\newblock


\bibitem[\protect\citeauthoryear{Aggarwal and Reddy}{Aggarwal and
  Reddy}{2014}]%
        {Aggarwal14}
\bibfield{author}{\bibinfo{person}{Charu~C Aggarwal} {and}
  \bibinfo{person}{Chandan~K Reddy}.} \bibinfo{year}{2014}\natexlab{}.
\newblock \bibinfo{booktitle}{\emph{Data Clustering: Algorithms and
  Applications}}.
\newblock \bibinfo{publisher}{CRC Press}, \bibinfo{address}{Boca Raton}.
\newblock


\bibitem[\protect\citeauthoryear{Albert and Barab{\'a}si}{Albert and
  Barab{\'a}si}{2002}]%
        {albert2002statistical}
\bibfield{author}{\bibinfo{person}{R{\'e}ka Albert} {and}
  \bibinfo{person}{Albert-L{\'a}szl{\'o} Barab{\'a}si}.}
  \bibinfo{year}{2002}\natexlab{}.
\newblock \showarticletitle{Statistical mechanics of complex networks}.
\newblock \bibinfo{journal}{\emph{Reviews of modern physics}}
  \bibinfo{volume}{74}, \bibinfo{number}{1} (\bibinfo{year}{2002}),
  \bibinfo{pages}{47}.
\newblock


\bibitem[\protect\citeauthoryear{Bakarov}{Bakarov}{2018}]%
        {corr/abs-1801-09536}
\bibfield{author}{\bibinfo{person}{Amir Bakarov}.}
  \bibinfo{year}{2018}\natexlab{}.
\newblock \showarticletitle{A Survey of Word Embeddings Evaluation Methods}.
\newblock \bibinfo{journal}{\emph{CoRR}}  \bibinfo{volume}{abs/1801.09536}
  (\bibinfo{year}{2018}).
\newblock


\bibitem[\protect\citeauthoryear{Belkin, Niyogi, and Sindhwani}{Belkin
  et~al\mbox{.}}{2006}]%
        {belkin2006manifold}
\bibfield{author}{\bibinfo{person}{M. Belkin}, \bibinfo{person}{P. Niyogi},
  {and} \bibinfo{person}{V Sindhwani}.} \bibinfo{year}{2006}\natexlab{}.
\newblock \showarticletitle{{Manifold regularization: A geometric framework for
  learning from labeled and unlabeled examples}}.
\newblock \bibinfo{journal}{\emph{Journal of Machine Learning Research}}
  \bibinfo{volume}{7}, \bibinfo{number}{1} (\bibinfo{year}{2006}),
  \bibinfo{pages}{2399--2434}.
\newblock


\bibitem[\protect\citeauthoryear{Blum, Lafferty, Rwebangira, and Reddy}{Blum
  et~al\mbox{.}}{2004}]%
        {blum2004semi}
\bibfield{author}{\bibinfo{person}{A. Blum}, \bibinfo{person}{J. Lafferty},
  \bibinfo{person}{M.R. Rwebangira}, {and} \bibinfo{person}{R Reddy}.}
  \bibinfo{year}{2004}\natexlab{}.
\newblock \showarticletitle{{Semi-supervised learning using randomized
  mincuts}}. In \bibinfo{booktitle}{\emph{ICML}}. \bibinfo{pages}{13}.
\newblock


\bibitem[\protect\citeauthoryear{Cao, Lu, and Xu}{Cao et~al\mbox{.}}{2016}]%
        {cao2016deep}
\bibfield{author}{\bibinfo{person}{Shaosheng Cao}, \bibinfo{person}{Wei Lu},
  {and} \bibinfo{person}{Qiongkai Xu}.} \bibinfo{year}{2016}\natexlab{}.
\newblock \showarticletitle{Deep Neural Networks for Learning Graph
  Representations}. In \bibinfo{booktitle}{\emph{AAAI}}.
  \bibinfo{pages}{1145--1152}.
\newblock


\bibitem[\protect\citeauthoryear{Chapelle, Weston, and Scholkopf}{Chapelle
  et~al\mbox{.}}{2003}]%
        {chapelle03}
\bibfield{author}{\bibinfo{person}{O. Chapelle}, \bibinfo{person}{J. Weston},
  {and} \bibinfo{person}{B Scholkopf}.} \bibinfo{year}{2003}\natexlab{}.
\newblock \showarticletitle{{Cluster kernels for semi-supervised learning}}. In
  \bibinfo{booktitle}{\emph{NeurIPS}}. \bibinfo{pages}{601--608}.
\newblock


\bibitem[\protect\citeauthoryear{Chapelle and Zien}{Chapelle and Zien}{2005}]%
        {chapelle2005semi}
\bibfield{author}{\bibinfo{person}{O. Chapelle} {and} \bibinfo{person}{A
  Zien}.} \bibinfo{year}{2005}\natexlab{}.
\newblock \showarticletitle{{Semi-supervised classification by low density
  separation}}. In \bibinfo{booktitle}{\emph{International Workshop on
  Artificial Intelligence and Statistics}}. \bibinfo{pages}{57--64}.
\newblock


\bibitem[\protect\citeauthoryear{Chen, Wei, Huang, Ding, and Li}{Chen
  et~al\mbox{.}}{2020}]%
        {icml/Chen20}
\bibfield{author}{\bibinfo{person}{Ming Chen}, \bibinfo{person}{Zhewei Wei},
  \bibinfo{person}{Zengfeng Huang}, \bibinfo{person}{Bolin Ding}, {and}
  \bibinfo{person}{Yaliang Li}.} \bibinfo{year}{2020}\natexlab{}.
\newblock \showarticletitle{Simple and Deep Graph Convolutional Networks}. In
  \bibinfo{booktitle}{\emph{ICML}}.
\newblock


\bibitem[\protect\citeauthoryear{Church and Hanks}{Church and Hanks}{1990}]%
        {church1990word}
\bibfield{author}{\bibinfo{person}{Kenneth~Ward Church} {and}
  \bibinfo{person}{Patrick Hanks}.} \bibinfo{year}{1990}\natexlab{}.
\newblock \showarticletitle{Word association norms, mutual information, and
  lexicography}.
\newblock \bibinfo{journal}{\emph{Computational Linguistics}}
  \bibinfo{volume}{16}, \bibinfo{number}{1} (\bibinfo{year}{1990}),
  \bibinfo{pages}{22--29}.
\newblock


\bibitem[\protect\citeauthoryear{Dai, Li, Tang, and Wang}{Dai
  et~al\mbox{.}}{2018}]%
        {dai2018adversarial}
\bibfield{author}{\bibinfo{person}{Quanyu Dai}, \bibinfo{person}{Qiang Li},
  \bibinfo{person}{Jian Tang}, {and} \bibinfo{person}{Dan Wang}.}
  \bibinfo{year}{2018}\natexlab{}.
\newblock \showarticletitle{Adversarial network embedding}. In
  \bibinfo{booktitle}{\emph{AAAI}}.
\newblock


\bibitem[\protect\citeauthoryear{Defferrard, Bresson, and
  Vandergheynst}{Defferrard et~al\mbox{.}}{2016}]%
        {defferrard2016convolutional}
\bibfield{author}{\bibinfo{person}{Micha{\"e}l Defferrard},
  \bibinfo{person}{Xavier Bresson}, {and} \bibinfo{person}{Pierre
  Vandergheynst}.} \bibinfo{year}{2016}\natexlab{}.
\newblock \showarticletitle{Convolutional neural networks on graphs with fast
  localized spectral filtering}. In \bibinfo{booktitle}{\emph{NeurIPS}}.
  \bibinfo{pages}{3844--3852}.
\newblock


\bibitem[\protect\citeauthoryear{Ekambaram, Fanti, Ayazifar, and
  Ramchandran}{Ekambaram et~al\mbox{.}}{2013}]%
        {ekambaram2013wavelet}
\bibfield{author}{\bibinfo{person}{Venkatesan~N Ekambaram},
  \bibinfo{person}{Giulia Fanti}, \bibinfo{person}{Babak Ayazifar}, {and}
  \bibinfo{person}{Kannan Ramchandran}.} \bibinfo{year}{2013}\natexlab{}.
\newblock \showarticletitle{Wavelet-regularized graph semi-supervised
  learning}. In \bibinfo{booktitle}{\emph{Global Conference on Signal and
  Information Processing}}. \bibinfo{pages}{423--426}.
\newblock


\bibitem[\protect\citeauthoryear{Feng, Zhang, Dong, Han, Luan, Xu, Yang,
  Kharlamov, and Tang}{Feng et~al\mbox{.}}{2020}]%
        {feng2020graph}
\bibfield{author}{\bibinfo{person}{Wenzheng Feng}, \bibinfo{person}{Jie Zhang},
  \bibinfo{person}{Yuxiao Dong}, \bibinfo{person}{Yu Han},
  \bibinfo{person}{Huanbo Luan}, \bibinfo{person}{Qian Xu},
  \bibinfo{person}{Qiang Yang}, \bibinfo{person}{Evgeny Kharlamov}, {and}
  \bibinfo{person}{Jie Tang}.} \bibinfo{year}{2020}\natexlab{}.
\newblock \showarticletitle{Graph Random Neural Networks for Semi-Supervised
  Learning on Graphs}.
\newblock \bibinfo{journal}{\emph{Advances in Neural Information Processing
  Systems}}  \bibinfo{volume}{33} (\bibinfo{year}{2020}).
\newblock


\bibitem[\protect\citeauthoryear{Fukunaga}{Fukunaga}{2013}]%
        {fukunaga2013introduction}
\bibfield{author}{\bibinfo{person}{Keinosuke Fukunaga}.}
  \bibinfo{year}{2013}\natexlab{}.
\newblock \bibinfo{booktitle}{\emph{Introduction to statistical pattern
  recognition}}.
\newblock \bibinfo{publisher}{Elsevier}.
\newblock


\bibitem[\protect\citeauthoryear{Girault, Gon{\c{c}}alves, Fleury, and
  Mor}{Girault et~al\mbox{.}}{2014}]%
        {girault2014semi}
\bibfield{author}{\bibinfo{person}{Benjamin Girault}, \bibinfo{person}{Paulo
  Gon{\c{c}}alves}, \bibinfo{person}{Eric Fleury}, {and}
  \bibinfo{person}{Arashpreet~Singh Mor}.} \bibinfo{year}{2014}\natexlab{}.
\newblock \showarticletitle{Semi-supervised learning for graph to signal
  mapping: A graph signal wiener filter interpretation}. In
  \bibinfo{booktitle}{\emph{Conference on Acoustics, Speech and Signal
  Processing}}. \bibinfo{pages}{1115--1119}.
\newblock


\bibitem[\protect\citeauthoryear{Grinstead and Snell}{Grinstead and
  Snell}{2012}]%
        {grinstead2012introduction}
\bibfield{author}{\bibinfo{person}{Charles~M Grinstead} {and}
  \bibinfo{person}{James~Laurie Snell}.} \bibinfo{year}{2012}\natexlab{}.
\newblock \bibinfo{booktitle}{\emph{Introduction to probability}}.
\newblock \bibinfo{publisher}{American Mathematical Soc}.
\newblock


\bibitem[\protect\citeauthoryear{Grover and Leskovec}{Grover and
  Leskovec}{2016}]%
        {grover2016node2vec}
\bibfield{author}{\bibinfo{person}{Aditya Grover} {and} \bibinfo{person}{Jure
  Leskovec}.} \bibinfo{year}{2016}\natexlab{}.
\newblock \showarticletitle{node2vec: Scalable feature learning for networks}.
  In \bibinfo{booktitle}{\emph{ACM SIGKDD}}. ACM, \bibinfo{pages}{855--864}.
\newblock


\bibitem[\protect\citeauthoryear{Hamilton, Ying, and Leskovec}{Hamilton
  et~al\mbox{.}}{2017}]%
        {hamilton2017inductive}
\bibfield{author}{\bibinfo{person}{Will Hamilton}, \bibinfo{person}{Zhitao
  Ying}, {and} \bibinfo{person}{Jure Leskovec}.}
  \bibinfo{year}{2017}\natexlab{}.
\newblock \showarticletitle{Inductive representation learning on large graphs}.
  In \bibinfo{booktitle}{\emph{NeurIPS}}. \bibinfo{pages}{1024--1034}.
\newblock


\bibitem[\protect\citeauthoryear{Hein and Maier}{Hein and Maier}{2007}]%
        {Hein07}
\bibfield{author}{\bibinfo{person}{Matthias Hein} {and} \bibinfo{person}{Markus
  Maier}.} \bibinfo{year}{2007}\natexlab{}.
\newblock \showarticletitle{Manifold Denoising}. In
  \bibinfo{booktitle}{\emph{NeurIPS}}. \bibinfo{pages}{561--568}.
\newblock


\bibitem[\protect\citeauthoryear{Howard, Zhu, Chen, Kalenichenko, Wang, Weyand,
  Andreetto, and Adam}{Howard et~al\mbox{.}}{2017}]%
        {howard2017mobilenets}
\bibfield{author}{\bibinfo{person}{Andrew~G Howard}, \bibinfo{person}{Menglong
  Zhu}, \bibinfo{person}{Bo Chen}, \bibinfo{person}{Dmitry Kalenichenko},
  \bibinfo{person}{Weijun Wang}, \bibinfo{person}{Tobias Weyand},
  \bibinfo{person}{Marco Andreetto}, {and} \bibinfo{person}{Hartwig Adam}.}
  \bibinfo{year}{2017}\natexlab{}.
\newblock \showarticletitle{Mobilenets: Efficient convolutional neural networks
  for mobile vision applications}.
\newblock \bibinfo{journal}{\emph{arXiv preprint arXiv:1704.04861}}
  (\bibinfo{year}{2017}).
\newblock


\bibitem[\protect\citeauthoryear{Huang, Li, and Hu}{Huang
  et~al\mbox{.}}{2017}]%
        {huang2017accelerated}
\bibfield{author}{\bibinfo{person}{Xiao Huang}, \bibinfo{person}{Jundong Li},
  {and} \bibinfo{person}{Xia Hu}.} \bibinfo{year}{2017}\natexlab{}.
\newblock \showarticletitle{Accelerated attributed network embedding}. In
  \bibinfo{booktitle}{\emph{SIAM ICDM}}. \bibinfo{pages}{633--641}.
\newblock


\bibitem[\protect\citeauthoryear{Kipf and Welling}{Kipf and Welling}{2016}]%
        {kipf2016variational}
\bibfield{author}{\bibinfo{person}{Thomas~N Kipf} {and} \bibinfo{person}{Max
  Welling}.} \bibinfo{year}{2016}\natexlab{}.
\newblock \showarticletitle{Variational Graph Auto-Encoders}.
\newblock \bibinfo{journal}{\emph{NIPS Workshop on Bayesian Deep Learning}}
  (\bibinfo{year}{2016}).
\newblock


\bibitem[\protect\citeauthoryear{Kipf and Welling}{Kipf and Welling}{2017}]%
        {kipf2016semi}
\bibfield{author}{\bibinfo{person}{Thomas~N Kipf} {and} \bibinfo{person}{Max
  Welling}.} \bibinfo{year}{2017}\natexlab{}.
\newblock \showarticletitle{Semi-supervised classification with graph
  convolutional networks}. In \bibinfo{booktitle}{\emph{ICLR}}.
\newblock


\bibitem[\protect\citeauthoryear{Kurokawa, Oki, and Nagao}{Kurokawa
  et~al\mbox{.}}{2017}]%
        {kurokawa2017multi}
\bibfield{author}{\bibinfo{person}{Takashi Kurokawa}, \bibinfo{person}{Taihei
  Oki}, {and} \bibinfo{person}{Hiromichi Nagao}.}
  \bibinfo{year}{2017}\natexlab{}.
\newblock \showarticletitle{Multi-dimensional Graph Fourier Transform}.
\newblock \bibinfo{journal}{\emph{arXiv preprint arXiv:1712.07811}}
  (\bibinfo{year}{2017}).
\newblock


\bibitem[\protect\citeauthoryear{Lang}{Lang}{1995}]%
        {Lang95}
\bibfield{author}{\bibinfo{person}{Ken Lang}.} \bibinfo{year}{1995}\natexlab{}.
\newblock \showarticletitle{Newsweeder: Learning to filter netnews}. In
  \bibinfo{booktitle}{\emph{ICML}}. \bibinfo{pages}{331--339}.
\newblock


\bibitem[\protect\citeauthoryear{Li, Han, and Wu}{Li et~al\mbox{.}}{2018}]%
        {li2018deeper}
\bibfield{author}{\bibinfo{person}{Qimai Li}, \bibinfo{person}{Zhichao Han},
  {and} \bibinfo{person}{Xiao-Ming Wu}.} \bibinfo{year}{2018}\natexlab{}.
\newblock \showarticletitle{Deeper Insights into Graph Convolutional Networks
  for Semi-Supervised Learning}. In \bibinfo{booktitle}{\emph{AAAI}}.
  \bibinfo{pages}{3538--3545}.
\newblock


\bibitem[\protect\citeauthoryear{Li, Wu, Liu, Zhang, and Guan}{Li
  et~al\mbox{.}}{2019}]%
        {li2019label}
\bibfield{author}{\bibinfo{person}{Qimai Li}, \bibinfo{person}{Xiao-Ming Wu},
  \bibinfo{person}{Han Liu}, \bibinfo{person}{Xiaotong Zhang}, {and}
  \bibinfo{person}{Zhichao Guan}.} \bibinfo{year}{2019}\natexlab{}.
\newblock \showarticletitle{Label Efficient Semi-Supervised Learning via Graph
  Filtering}. In \bibinfo{booktitle}{\emph{CVPR}}. \bibinfo{pages}{9582--9591}.
\newblock


\bibitem[\protect\citeauthoryear{Li, Tarlow, Brockschmidt, and Zemel}{Li
  et~al\mbox{.}}{2016}]%
        {LiTBZ15}
\bibfield{author}{\bibinfo{person}{Yujia Li}, \bibinfo{person}{Daniel Tarlow},
  \bibinfo{person}{Marc Brockschmidt}, {and} \bibinfo{person}{Richard~S.
  Zemel}.} \bibinfo{year}{2016}\natexlab{}.
\newblock \showarticletitle{Gated Graph Sequence Neural Networks}. In
  \bibinfo{booktitle}{\emph{ICLR}}.
\newblock


\bibitem[\protect\citeauthoryear{Liao, Zhao, Urtasun, and Zemel}{Liao
  et~al\mbox{.}}{2019}]%
        {liao2019lanczosnet}
\bibfield{author}{\bibinfo{person}{Renjie Liao}, \bibinfo{person}{Zhizhen
  Zhao}, \bibinfo{person}{Raquel Urtasun}, {and} \bibinfo{person}{Richard~S
  Zemel}.} \bibinfo{year}{2019}\natexlab{}.
\newblock \showarticletitle{LanczosNet: Multi-Scale Deep Graph Convolutional
  Networks}.
\newblock \bibinfo{journal}{\emph{CoRR}}  \bibinfo{volume}{abs/1901.01484}
  (\bibinfo{year}{2019}).
\newblock


\bibitem[\protect\citeauthoryear{McCallumzy, Nigamy, Renniey, and
  Seymorey}{McCallumzy et~al\mbox{.}}{1999}]%
        {mccallumzy1999building}
\bibfield{author}{\bibinfo{person}{Andrew McCallumzy}, \bibinfo{person}{Kamal
  Nigamy}, \bibinfo{person}{Jason Renniey}, {and} \bibinfo{person}{Kristie
  Seymorey}.} \bibinfo{year}{1999}\natexlab{}.
\newblock \showarticletitle{Building domain-specific search engines with
  machine learning techniques}. In \bibinfo{booktitle}{\emph{Proceedings of the
  AAAI Spring Symposium on Intelligent Agents in Cyberspace}}. Citeseer,
  \bibinfo{pages}{28--39}.
\newblock


\bibitem[\protect\citeauthoryear{Monti, Bronstein, and Bresson}{Monti
  et~al\mbox{.}}{[n.d.]}]%
        {monti2017geometric}
\bibfield{author}{\bibinfo{person}{Federico Monti}, \bibinfo{person}{Michael~M.
  Bronstein}, {and} \bibinfo{person}{Xavier Bresson}.}
  \bibinfo{year}{[n.d.]}\natexlab{}.
\newblock \showarticletitle{Geometric Matrix Completion with Recurrent
  Multi-Graph Neural Networks}. In \bibinfo{booktitle}{\emph{NIPS'17}}.
\newblock


\bibitem[\protect\citeauthoryear{NT and Maehara}{NT and Maehara}{2021}]%
        {nt2019revisiting}
\bibfield{author}{\bibinfo{person}{Hoang NT} {and} \bibinfo{person}{Takanori
  Maehara}.} \bibinfo{year}{2021}\natexlab{}.
\newblock \showarticletitle{Revisiting Graph Neural Networks: All We Have is
  Low-Pass Filters}. In \bibinfo{booktitle}{\emph{25th International Conference
  on Pattern Recognition}} \emph{(\bibinfo{series}{ICPR'20})}.
\newblock


\bibitem[\protect\citeauthoryear{Pan, Hu, Long, Jiang, Yao, and Zhang}{Pan
  et~al\mbox{.}}{2018}]%
        {pan2018adversarially}
\bibfield{author}{\bibinfo{person}{Shirui Pan}, \bibinfo{person}{Ruiqi Hu},
  \bibinfo{person}{Guodong Long}, \bibinfo{person}{Jing Jiang},
  \bibinfo{person}{Lina Yao}, {and} \bibinfo{person}{Chengqi Zhang}.}
  \bibinfo{year}{2018}\natexlab{}.
\newblock \showarticletitle{Adversarially Regularized Graph Autoencoder for
  Graph Embedding}. In \bibinfo{booktitle}{\emph{IJCAI}}.
  \bibinfo{pages}{2609--2615}.
\newblock


\bibitem[\protect\citeauthoryear{Pan, Wu, Zhu, Zhang, and Wang}{Pan
  et~al\mbox{.}}{2016}]%
        {PanWZZW16}
\bibfield{author}{\bibinfo{person}{Shirui Pan}, \bibinfo{person}{Jia Wu},
  \bibinfo{person}{Xingquan Zhu}, \bibinfo{person}{Chengqi Zhang}, {and}
  \bibinfo{person}{Yang Wang}.} \bibinfo{year}{2016}\natexlab{}.
\newblock \showarticletitle{Tri-Party Deep Network Representation}. In
  \bibinfo{booktitle}{\emph{IJCAI}}. \bibinfo{pages}{1895--1901}.
\newblock


\bibitem[\protect\citeauthoryear{Pei, Yu, Yu, and Zhang}{Pei
  et~al\mbox{.}}{[n.d.]}]%
        {DBLP:conf/kdd/Pei0Y020}
\bibfield{author}{\bibinfo{person}{Shichao Pei}, \bibinfo{person}{Lu Yu},
  \bibinfo{person}{Guoxian Yu}, {and} \bibinfo{person}{Xiangliang Zhang}.}
  \bibinfo{year}{[n.d.]}\natexlab{}.
\newblock \showarticletitle{{REA:} Robust Cross-lingual Entity Alignment
  Between Knowledge Graphs}. In \bibinfo{booktitle}{\emph{KDD}}.
\newblock


\bibitem[\protect\citeauthoryear{Peng, Huang, Luo, Zheng, Rong, Xu, and
  Huang}{Peng et~al\mbox{.}}{2020}]%
        {www/PengHLZRXH20}
\bibfield{author}{\bibinfo{person}{Zhen Peng}, \bibinfo{person}{Wenbing Huang},
  \bibinfo{person}{Minnan Luo}, \bibinfo{person}{Qinghua Zheng},
  \bibinfo{person}{Yu Rong}, \bibinfo{person}{Tingyang Xu}, {and}
  \bibinfo{person}{Junzhou Huang}.} \bibinfo{year}{2020}\natexlab{}.
\newblock \showarticletitle{Graph Representation Learning via Graphical Mutual
  Information Maximization}. In \bibinfo{booktitle}{\emph{{WWW}}}.
  \bibinfo{pages}{259--270}.
\newblock


\bibitem[\protect\citeauthoryear{Pennington, Socher, and Manning}{Pennington
  et~al\mbox{.}}{2014}]%
        {pennington2014glove}
\bibfield{author}{\bibinfo{person}{Jeffrey Pennington},
  \bibinfo{person}{Richard Socher}, {and} \bibinfo{person}{Christopher
  Manning}.} \bibinfo{year}{2014}\natexlab{}.
\newblock \showarticletitle{Glove: Global vectors for word representation}. In
  \bibinfo{booktitle}{\emph{EMNLP}}. \bibinfo{pages}{1532--1543}.
\newblock


\bibitem[\protect\citeauthoryear{Perona and Freeman}{Perona and
  Freeman}{1998}]%
        {perona1998}
\bibfield{author}{\bibinfo{person}{Pietro Perona} {and}
  \bibinfo{person}{William Freeman}.} \bibinfo{year}{1998}\natexlab{}.
\newblock \showarticletitle{A factorization approach to grouping}. In
  \bibinfo{booktitle}{\emph{ECCV}}. \bibinfo{pages}{655--670}.
\newblock


\bibitem[\protect\citeauthoryear{Perozzi, Al-Rfou, and Skiena}{Perozzi
  et~al\mbox{.}}{2014}]%
        {perozzi2014deepwalk}
\bibfield{author}{\bibinfo{person}{Bryan Perozzi}, \bibinfo{person}{Rami
  Al-Rfou}, {and} \bibinfo{person}{Steven Skiena}.}
  \bibinfo{year}{2014}\natexlab{}.
\newblock \showarticletitle{Deepwalk: Online learning of social
  representations}. In \bibinfo{booktitle}{\emph{{ACM} {SIGKDD}}}.
  \bibinfo{pages}{701--710}.
\newblock


\bibitem[\protect\citeauthoryear{Qu, Bengio, and Tang}{Qu
  et~al\mbox{.}}{2019}]%
        {icml/QuBT19}
\bibfield{author}{\bibinfo{person}{Meng Qu}, \bibinfo{person}{Yoshua Bengio},
  {and} \bibinfo{person}{Jian Tang}.} \bibinfo{year}{2019}\natexlab{}.
\newblock \showarticletitle{{GMNN:} Graph Markov Neural Networks}. In
  \bibinfo{booktitle}{\emph{{ICML}}}. \bibinfo{pages}{5241--5250}.
\newblock


\bibitem[\protect\citeauthoryear{Sacc{\'a}, Diligenti, and Gori}{Sacc{\'a}
  et~al\mbox{.}}{2013}]%
        {sacca2013collective}
\bibfield{author}{\bibinfo{person}{Claudio Sacc{\'a}},
  \bibinfo{person}{Michelangelo Diligenti}, {and} \bibinfo{person}{Marco
  Gori}.} \bibinfo{year}{2013}\natexlab{}.
\newblock \showarticletitle{Collective Classification Using Semantic Based
  Regularization}. In \bibinfo{booktitle}{\emph{IEEE ICMLA}},
  Vol.~\bibinfo{volume}{1}. \bibinfo{pages}{283--286}.
\newblock


\bibitem[\protect\citeauthoryear{Sen, Namata, Bilgic, Getoor, Galligher, and
  Eliassi-Rad}{Sen et~al\mbox{.}}{2008}]%
        {sen2008collective}
\bibfield{author}{\bibinfo{person}{Prithviraj Sen}, \bibinfo{person}{Galileo
  Namata}, \bibinfo{person}{Mustafa Bilgic}, \bibinfo{person}{Lise Getoor},
  \bibinfo{person}{Brian Galligher}, {and} \bibinfo{person}{Tina Eliassi-Rad}.}
  \bibinfo{year}{2008}\natexlab{}.
\newblock \showarticletitle{Collective classification in network data}.
\newblock \bibinfo{journal}{\emph{AI Magazine}} \bibinfo{volume}{29},
  \bibinfo{number}{3} (\bibinfo{year}{2008}), \bibinfo{pages}{93--106}.
\newblock


\bibitem[\protect\citeauthoryear{Talukdar and Crammer}{Talukdar and
  Crammer}{2009}]%
        {talukdar2009new}
\bibfield{author}{\bibinfo{person}{Partha~Pratim Talukdar} {and}
  \bibinfo{person}{Koby Crammer}.} \bibinfo{year}{2009}\natexlab{}.
\newblock \showarticletitle{New regularized algorithms for transductive
  learning}. In \bibinfo{booktitle}{\emph{Joint European Conference on Machine
  Learning and Knowledge Discovery in Databases}}. Springer,
  \bibinfo{pages}{442--457}.
\newblock


\bibitem[\protect\citeauthoryear{Velickovic, Cucurull, Casanova, Romero,
  Li{\`{o}}, and Bengio}{Velickovic et~al\mbox{.}}{2018}]%
        {velickovic2017graph}
\bibfield{author}{\bibinfo{person}{Petar Velickovic}, \bibinfo{person}{Guillem
  Cucurull}, \bibinfo{person}{Arantxa Casanova}, \bibinfo{person}{Adriana
  Romero}, \bibinfo{person}{Pietro Li{\`{o}}}, {and} \bibinfo{person}{Yoshua
  Bengio}.} \bibinfo{year}{2018}\natexlab{}.
\newblock \showarticletitle{Graph Attention Networks}. In
  \bibinfo{booktitle}{\emph{ICLR}}.
\newblock


\bibitem[\protect\citeauthoryear{Velickovic, Fedus, Hamilton, Li{\`{o}},
  Bengio, and Hjelm}{Velickovic et~al\mbox{.}}{2019}]%
        {velivckovic2018deep}
\bibfield{author}{\bibinfo{person}{Petar Velickovic}, \bibinfo{person}{William
  Fedus}, \bibinfo{person}{William~L. Hamilton}, \bibinfo{person}{Pietro
  Li{\`{o}}}, \bibinfo{person}{Yoshua Bengio}, {and} \bibinfo{person}{R.~Devon
  Hjelm}.} \bibinfo{year}{2019}\natexlab{}.
\newblock \showarticletitle{Deep Graph Infomax}. In
  \bibinfo{booktitle}{\emph{{ICLR}}}.
\newblock


\bibitem[\protect\citeauthoryear{Von~Luxburg}{Von~Luxburg}{2007}]%
        {von2007tutorial}
\bibfield{author}{\bibinfo{person}{Ulrike Von~Luxburg}.}
  \bibinfo{year}{2007}\natexlab{}.
\newblock \showarticletitle{A tutorial on spectral clustering}.
\newblock \bibinfo{journal}{\emph{Statistics and computing}}
  \bibinfo{volume}{17}, \bibinfo{number}{4} (\bibinfo{year}{2007}),
  \bibinfo{pages}{395--416}.
\newblock


\bibitem[\protect\citeauthoryear{Wang, Pan, Long, Zhu, and Jiang}{Wang
  et~al\mbox{.}}{2017}]%
        {wang2017mgae}
\bibfield{author}{\bibinfo{person}{Chun Wang}, \bibinfo{person}{Shirui Pan},
  \bibinfo{person}{Guodong Long}, \bibinfo{person}{Xingquan Zhu}, {and}
  \bibinfo{person}{Jing Jiang}.} \bibinfo{year}{2017}\natexlab{}.
\newblock \showarticletitle{Mgae: Marginalized graph autoencoder for graph
  clustering}. In \bibinfo{booktitle}{\emph{CIKM}}. \bibinfo{pages}{889--898}.
\newblock


\bibitem[\protect\citeauthoryear{Wang, Cui, and Zhu}{Wang
  et~al\mbox{.}}{2016}]%
        {wang2016structural}
\bibfield{author}{\bibinfo{person}{Daixin Wang}, \bibinfo{person}{Peng Cui},
  {and} \bibinfo{person}{Wenwu Zhu}.} \bibinfo{year}{2016}\natexlab{}.
\newblock \showarticletitle{Structural deep network embedding}. In
  \bibinfo{booktitle}{\emph{{ACM} {SIGKDD}}}. \bibinfo{pages}{1225--1234}.
\newblock


\bibitem[\protect\citeauthoryear{Wang, Liu, Jiang, Li, and Fu}{Wang
  et~al\mbox{.}}{2020}]%
        {DBLP:conf/kdd/WangLJLF20}
\bibfield{author}{\bibinfo{person}{Pengyang Wang}, \bibinfo{person}{Kunpeng
  Liu}, \bibinfo{person}{Lu Jiang}, \bibinfo{person}{Xiaolin Li}, {and}
  \bibinfo{person}{Yanjie Fu}.} \bibinfo{year}{2020}\natexlab{}.
\newblock \showarticletitle{Incremental Mobile User Profiling: Reinforcement
  Learning with Spatial Knowledge Graph for Modeling Event Streams}. In
  \bibinfo{booktitle}{\emph{KDD}}. \bibinfo{pages}{853--861}.
\newblock


\bibitem[\protect\citeauthoryear{West and Leskovec}{West and Leskovec}{2012}]%
        {west2012human}
\bibfield{author}{\bibinfo{person}{Robert West} {and} \bibinfo{person}{Jure
  Leskovec}.} \bibinfo{year}{2012}\natexlab{}.
\newblock \showarticletitle{Human wayfinding in information networks}. In
  \bibinfo{booktitle}{\emph{WWW}}. \bibinfo{pages}{619--628}.
\newblock


\bibitem[\protect\citeauthoryear{West, Pineau, and Precup}{West
  et~al\mbox{.}}{2009}]%
        {west2009wikispeedia}
\bibfield{author}{\bibinfo{person}{Robert West}, \bibinfo{person}{Joelle
  Pineau}, {and} \bibinfo{person}{Doina Precup}.}
  \bibinfo{year}{2009}\natexlab{}.
\newblock \showarticletitle{Wikispeedia: An Online Game for Inferring Semantic
  Distances between Concepts}. In \bibinfo{booktitle}{\emph{IJCAI}}.
  \bibinfo{pages}{1598--1603}.
\newblock


\bibitem[\protect\citeauthoryear{Weston, Ratle, Mobahi, and Collobert}{Weston
  et~al\mbox{.}}{2008}]%
        {weston2012deep}
\bibfield{author}{\bibinfo{person}{Jason Weston},
  \bibinfo{person}{Fr{\'e}d{\'e}ric Ratle}, \bibinfo{person}{Hossein Mobahi},
  {and} \bibinfo{person}{Ronan Collobert}.} \bibinfo{year}{2008}\natexlab{}.
\newblock \showarticletitle{Deep learning via semi-supervised embedding}. In
  \bibinfo{booktitle}{\emph{ICML}}. \bibinfo{pages}{1168--1175}.
\newblock


\bibitem[\protect\citeauthoryear{Wu, Zhao, Wang, and Pan}{Wu
  et~al\mbox{.}}{[n.d.]}]%
        {DBLP:conf/kdd/WuZWP20}
\bibfield{author}{\bibinfo{person}{Ning Wu}, \bibinfo{person}{Wayne~Xin Zhao},
  \bibinfo{person}{Jingyuan Wang}, {and} \bibinfo{person}{Dayan Pan}.}
  \bibinfo{year}{[n.d.]}\natexlab{}.
\newblock


\bibitem[\protect\citeauthoryear{Wu, Li, So, Wright, and Chang}{Wu
  et~al\mbox{.}}{2012}]%
        {Wu12parw}
\bibfield{author}{\bibinfo{person}{Xiao-Ming Wu}, \bibinfo{person}{Zhenguo Li},
  \bibinfo{person}{Anthony~M. So}, \bibinfo{person}{John Wright}, {and}
  \bibinfo{person}{Shih-fu Chang}.} \bibinfo{year}{2012}\natexlab{}.
\newblock \showarticletitle{{Learning with Partially Absorbing Random Walks}}.
  In \bibinfo{booktitle}{\emph{NeurIPS}}. \bibinfo{pages}{3077--3085}.
\newblock


\bibitem[\protect\citeauthoryear{Xu, Li, Tian, Sonobe, Kawarabayashi, and
  Jegelka}{Xu et~al\mbox{.}}{2018}]%
        {xu2018representation}
\bibfield{author}{\bibinfo{person}{Keyulu Xu}, \bibinfo{person}{Chengtao Li},
  \bibinfo{person}{Yonglong Tian}, \bibinfo{person}{Tomohiro Sonobe},
  \bibinfo{person}{Ken-ichi Kawarabayashi}, {and} \bibinfo{person}{Stefanie
  Jegelka}.} \bibinfo{year}{2018}\natexlab{}.
\newblock \showarticletitle{Representation Learning on Graphs with Jumping
  Knowledge Networks}. In \bibinfo{booktitle}{\emph{Proceedings of the 35th
  ICML}} \emph{(\bibinfo{series}{Proceedings of Machine Learning Research},
  Vol.~\bibinfo{volume}{80})}, \bibfield{editor}{\bibinfo{person}{Jennifer Dy}
  {and} \bibinfo{person}{Andreas Krause}} (Eds.). \bibinfo{publisher}{PMLR},
  \bibinfo{pages}{5453--5462}.
\newblock


\bibitem[\protect\citeauthoryear{Yang, Liu, Zhao, Sun, and Chang}{Yang
  et~al\mbox{.}}{2015}]%
        {yang2015network}
\bibfield{author}{\bibinfo{person}{Cheng Yang}, \bibinfo{person}{Zhiyuan Liu},
  \bibinfo{person}{Deli Zhao}, \bibinfo{person}{Maosong Sun}, {and}
  \bibinfo{person}{Edward~Y Chang}.} \bibinfo{year}{2015}\natexlab{}.
\newblock \showarticletitle{Network representation learning with rich text
  information}. In \bibinfo{booktitle}{\emph{IJCAI}}.
  \bibinfo{pages}{2111--2117}.
\newblock


\bibitem[\protect\citeauthoryear{Yang, Cohen, and Salakhutdinov}{Yang
  et~al\mbox{.}}{2016}]%
        {yang2016revisiting}
\bibfield{author}{\bibinfo{person}{Zhilin Yang}, \bibinfo{person}{William~W
  Cohen}, {and} \bibinfo{person}{Ruslan Salakhutdinov}.}
  \bibinfo{year}{2016}\natexlab{}.
\newblock \showarticletitle{Revisiting semi-supervised learning with graph
  embeddings}. In \bibinfo{booktitle}{\emph{ICML}}. \bibinfo{pages}{40--48}.
\newblock


\bibitem[\protect\citeauthoryear{Ye, Chen, and Zheng}{Ye et~al\mbox{.}}{2018}]%
        {ye2018deep}
\bibfield{author}{\bibinfo{person}{Fanghua Ye}, \bibinfo{person}{Chuan Chen},
  {and} \bibinfo{person}{Zibin Zheng}.} \bibinfo{year}{2018}\natexlab{}.
\newblock \showarticletitle{Deep Autoencoder-like Nonnegative Matrix
  Factorization for Community Detection}. In \bibinfo{booktitle}{\emph{CIKM}}.
  \bibinfo{pages}{1393--1402}.
\newblock


\bibitem[\protect\citeauthoryear{Zhang, Shi, Xie, Ma, King, and Yeung}{Zhang
  et~al\mbox{.}}{2018}]%
        {ZhangSXMKY18}
\bibfield{author}{\bibinfo{person}{Jiani Zhang}, \bibinfo{person}{Xingjian
  Shi}, \bibinfo{person}{Junyuan Xie}, \bibinfo{person}{Hao Ma},
  \bibinfo{person}{Irwin King}, {and} \bibinfo{person}{Dit{-}Yan Yeung}.}
  \bibinfo{year}{2018}\natexlab{}.
\newblock \showarticletitle{GaAN: Gated Attention Networks for Learning on
  Large and Spatiotemporal Graphs}. In \bibinfo{booktitle}{\emph{UAI}}.
  \bibinfo{pages}{339--349}.
\newblock


\bibitem[\protect\citeauthoryear{Zhang and Ando}{Zhang and Ando}{2006}]%
        {zhang06}
\bibfield{author}{\bibinfo{person}{T. Zhang} {and} \bibinfo{person}{R.K Ando}.}
  \bibinfo{year}{2006}\natexlab{}.
\newblock \showarticletitle{{Analysis of spectral kernel design based
  semi-supervised learning}}. In \bibinfo{booktitle}{\emph{NeurIPS}}.
  \bibinfo{pages}{1601--1608}.
\newblock


\bibitem[\protect\citeauthoryear{Zhang, Liu, Li, and Wu}{Zhang
  et~al\mbox{.}}{2019}]%
        {ijcai/0003LLW19}
\bibfield{author}{\bibinfo{person}{Xiaotong Zhang}, \bibinfo{person}{Han Liu},
  \bibinfo{person}{Qimai Li}, {and} \bibinfo{person}{Xiao{-}Ming Wu}.}
  \bibinfo{year}{2019}\natexlab{}.
\newblock \showarticletitle{Attributed Graph Clustering via Adaptive Graph
  Convolution}. In \bibinfo{booktitle}{\emph{{IJCAI}}}.
  \bibinfo{pages}{4327--4333}.
\newblock


\bibitem[\protect\citeauthoryear{Zhou, Bousquet, Lal, Weston, and
  Sch{\"o}lkopf}{Zhou et~al\mbox{.}}{2004}]%
        {Zhou03}
\bibfield{author}{\bibinfo{person}{Denny Zhou}, \bibinfo{person}{Olivier
  Bousquet}, \bibinfo{person}{Thomas~N Lal}, \bibinfo{person}{Jason Weston},
  {and} \bibinfo{person}{Bernhard Sch{\"o}lkopf}.}
  \bibinfo{year}{2004}\natexlab{}.
\newblock \showarticletitle{{Learning with local and global consistency}}. In
  \bibinfo{booktitle}{\emph{NeurIPS}}. \bibinfo{pages}{321--328}.
\newblock


\bibitem[\protect\citeauthoryear{Zhu, Ghahramani, and Lafferty}{Zhu
  et~al\mbox{.}}{2003}]%
        {Zhu03}
\bibfield{author}{\bibinfo{person}{Xiaojin Zhu}, \bibinfo{person}{Zoubin
  Ghahramani}, {and} \bibinfo{person}{John~D Lafferty}.}
  \bibinfo{year}{2003}\natexlab{}.
\newblock \showarticletitle{Semi-supervised learning using gaussian fields and
  harmonic functions}. In \bibinfo{booktitle}{\emph{ICML}}.
  \bibinfo{pages}{912--919}.
\newblock


\end{thebibliography}

\newpage
\appendix
\clearpage
\appendices



\section{Experimental Settings}\label{ap:experiment}

\subsection{Semi-supervised Node Classification}


For constructing the PPMI graph, we set the context window size to 20 and use the inverse distance as co-occurrence weight. For constructing the Emb graph, we use GloVe \cite{pennington2014glove} to learn word embeddings and set the number of nearest neighbours $k=20$. The classifier of DSGC we use is a multi-layer perceptron (MLP) with a 64-unit hidden layer. The MLP is trained for 200 epochs by Adam Optimizer.

Hyperparameters of all models, including our methods and baselines, are tuned by grid search based on validation. Dropout rate is selected from \{0, 0.2, 0.5\}, and 0.2 is chosen. Learning rate is selected from \{1, 0.1, 0.01, 0.001\}, and 0.1 is chosen. Weight decay is selected from \{5e-3, 5e-4, 5e-5, 5e-6, 5e-7\}, and 5e-4 is chosen for WebKB, 5e-5 for L-Cora, 5e-6 for 20 NG, 5e-7 for Wiki. Results of our methods and baselines are averaged over 50 runs.

\subsection{Node Clustering}
For our method DSGC, the attribute affinity graphs PPMI and Emb are constructed in the same way as in object classification. For other baselines, we follow the parameter settings described in the original papers. In particular, for GAE and VGAE \cite{kipf2016variational}, we construct encoders with a 32-neuron hidden layer and a 16-neuron embedding layer, and train the encoders for 200 iterations using the Adam algorithm with a learning rate of 0.01.

For MGAE \cite{wang2017mgae}, the corruption level $p$ is 0.4, the number of layers is 3, and the parameter $\lambda$ is $10^{-5}$. For ARGE and ARVGE \cite{pan2018adversarially}, we construct encoders with a 32-neuron hidden layer and a 16-neuron embedding layer. The discriminators are built by two hidden layers with 16 neurons and 64 neurons respectively. We train all the autoencoder-related models for 200 iterations and optimize them using the Adam algorithm. The learning rates of encoder and discriminator are both 0.001. For AGC \cite{ijcai/0003LLW19}, the maximum iteration number is 60. For fair comparison, these baselines also adopt spectral clustering as our DSGC to obtain the clustering results. We repeat each method for 10 times and report the average clustering results and standard deviations.

\setcounter{theorem}{0}

\section{}\label{ap:theorem1}
Assume that objects from the same class are connected with probability $r$, and objects from different classes are connected with probability $q$, i.e., the adjacency matrix $\bm A^{(1)}$ of object graph obeys following distribution:

\vskip 0.05in \hfil{\centering \def\arraystretch{1.2}
\begin{tabular}{c|cc}
                           & if $y_i=y_j$ & if $y_i\not=y_j$ \\
     \hline
     $\Pr(a_{ij} \not= 0)$ & $r$       & $q$ \\
     $\Pr(a_{ij} = 0)$     & $1-r$     & $1-q$ 
\end{tabular}\vspace{0.05in}} \hfil \\ 
We also assume that classes are balanced, i.e., $\Pr(\Y=k)=1/K$ for all $k$. Then, with the stochastic graph filter $\bm G={\bm D^{-1} \bm A^{(1)}}$, we have the following theorem.
\begin{theorem}
When $q$ is sufficiently small, the $\IntraVar/\InterVar$ ratio of $\Z$ is less than or equal to that of $\X$, i.e.,
	\begin{align}
	\frac{\E{\Var{\Z|\Y}}}{\Var{\E{\Z|\Y}}} \le \frac{\E{\Var{\X|\Y}}}{\Var{\E{\X|\Y}}}.
	\end{align}
\end{theorem}

\begin{proof}
The proof consists of two parts. In the first part, we prove that inter-class variance is unchanged after object graph convolution, when $q$ approximates 0, i.e.,
\begin{equation}
        \lim_{q\to0}\Var{\E{\Z|\Y}} = \Var{\E{\X|\Y}}.
\end{equation}
In the second part, we prove that intra-class variance becomes smaller after object graph convolution, i.e.,
\begin{equation}
        \E{\Var{\Z|\Y}} \le \E{\Var{\X|\Y}},
\end{equation}
when $G$ is a stochastic matrix.

\noindent\textbf{Part 1. Inter-class variance is unchanged.}\quad
Since $\z_i = \sum_{j} G_{ij} \x_j$,
we have
\begin{flalign*}
    &\quad \E{\z_i|y_i=k}
    = \sum_{j} \E{G_{ij}} \E{\x_j} & \\ 
    &= \sum_{j, y_j=k} \E{G_{ij}} \E{\x_j} + \sum_{j, y_j\not=k} \E{G_{ij}} \E{\x_j}\\
    &= \frac{\sum_{j, y_j=k} \E{a_{ij}} \E{\x_j} + \sum_{j, y_j\not=k} \E{a_{ij}} \E{\x_j}}{\sum_j \E{a_{ij}}}\\
    &= \frac{r \sum_{j, y_j=k} \E{\X|\Y=k} + q \sum_{j, y_j\not=k} \E{\x_j} }{\frac{N}{K} (r-q) + Nq}\\
    &= \frac{\frac{N}{K} r \E{\X|\Y=k} + q \sum_{j} \E{\x_j} - q \sum_{j, y_j=k} \E{\x_j} } {\frac{N}{K} (r-q) + Nq}\\
    &= \frac{\frac{N}{K} (r-q) \E{\X|\Y=k} + Nq \E{\X}} {\frac{N}{K} (r-q) + Nq}\\
    &= \frac{(r-q) \E{\X|\Y=k} + Kq \E{\X}} {(r-q) + Kq} \stepcounter{equation}\tag{\theequation} \label{eq:ez}
\end{flalign*}
When $q$ approximates $0$, Eq.~(\ref{eq:ez}) approximates $\E{\X|\Y=k}$, so
\begin{flalign*}
    &\quad \E{\Z|\Y=k}
    = \sum_{i, y_i=k} \Pr(\Z=\z_i|y_i=k) \E{\z_i|y_i=k} & \\
    &= \sum_{i, y_i=k} \Pr(\Z=\z_i|y_i=k) \E{\X|\Y=k} = \E{\X|\Y=k}.
\end{flalign*}
Take variance of both side, we get
\begin{equation}\label{eq:inter}
    \Var{\E{\Z|\Y}} = \Var{\E{\X|\Y}}.
\end{equation}

\noindent\textbf{Part 2. Intra-class variance becomes smaller.}\quad
Denote by $\Cov{\cdot,\cdot}$ the covariance of two random variables. We have following inequality about variance.
\begin{flalign*}
    \Var{\sum_j G_{ij}\x_j}
    &= \sum_j G_{ij}^2 \Var{\x_j} + \sum_{j,l}G_{ij}G_{il}\Cov{\x_j,\x_l}  & \\ 
    &\le \sum_{j,l}G_{ij}G_{il}\sqrt{\Var{\x_j}}\sqrt{\Var{\x_l}}           & \\ 
    &= \left(\sum_j G_{ij}\sqrt{\Var{\x_j}}\right)^2. \stepcounter{equation}\tag{\theequation} \label{eq:th1p2in}
\end{flalign*}
Consider the variance of filtering result $\bm z_i$ for each sample in class $k$, it is less than variance of $\X$ of that class:
\begin{flalign*}
    &\quad \Var{\z_i|y_i=k} \\
    &= \Var{\left.\sum_{j} G_{ij} \x_j \right|y_j=k}                                     && \hphantom{\text{property of covariance}}\\
    &\le \left(\sum_{j} G_{ij}\sqrt{\Var{\x_j|y_j=k}}\right)^2              &&\text{\# by inequality~(\ref{eq:th1p2in})}\\
    &= \left(\sum_{j} G_{ij}\sqrt{\Var{\X|\Y=k}}\right)^2 \\
    &= \left(\sqrt{\Var{\X|\Y=k}}\right)^2                                  &&\text{\# since $\sum_j G_{ij}=1$}\\
    &= \Var{\X|\Y=k},
\end{flalign*}
Then variance of random vector $\Z$ for each class is less than variance of $\X$ of that class:
\begin{flalign*}
    \Var{\Z|\Y=k} &= \sum_{i,y_i=k} \Pr(\Z=\z_i|y_i=k)\Var{\z_i|y_i=k} &\\
    &\le \Var{\X|\Y=k}.
\end{flalign*}
Sum tham over all classes:
\begin{flalign}
    \qquad \E{\Var{\Z|\Y}}
    &= \sum_{k} \Pr(\Y=k)\Var{\Z|\Y=k} & \nonumber\\
    &\le \sum_{k} \Pr(\Y=k)\Var{\X|\Y=k} \nonumber\\
    &= \E{\Var{\X|\Y}}. \label{eq:intra}
\end{flalign}
Combining Eq. (\ref{eq:inter}) and Eq. (\ref{eq:intra}), we prove that when $q$ is sufficiently small,
\begin{equation}
        \frac{\E{\Var{\Z|\Y}}}{\Var{\E{\Z|\Y}}} \le \frac{\E{\Var{\X|\Y}}}{\Var{\E{\X|\Y}}}.
\end{equation}
\end{proof}

\section{}\label{ap:theorem23}
\begin{theorem}
    If the attribute graph convolutional filter $\bm F$ is a doubly stochastic matrix, then the output of attribute graph convolution has an intra-class variance less than or equal to that of \;$\X$, i.e.,
\begin{align*}
\sum\nolimits_i F_{ij} = \sum\nolimits_j &F_{ij} = 1 \;\text{and}\; F_{ij}\ge0,\forall\;i,j \\
    \qquad&\Rightarrow\quad \E{\Var{\bm F^\top\X|\Y}} \le \E{\Var{\X|\Y}}.
\end{align*}
\end{theorem}

\begin{proof}
    We first prove a lemma that variance of each class will not increase after attribute graph convolution, i.e., $\Var{F^\top\X|\Y=k} \le \Var{\X|\Y=k}$. Denote by $\Cov{\cdot}$ the covariance matrix of a random vector. Based on our definition of variance at the beginning of section \ref{sec:analysis}, we have
    \begin{flalign*}
        &\quad \Var{\bm F^\top\X|\Y=k} \\
        & = \Tr{\Cov{\bm F^\top\X|\Y=k}} \\
        & = \Tr{\bm F^\top\Cov{\X|\Y=k}\bm F}                               && \text{\# property of covariance} \\
        & = \Tr{\Cov{\X|\Y=k} \bm F\bm F^\top}                              && \text{\# cyclic property of trace} \\
        & = \sum_{ij}\Cov{\X_i, \X_j|\Y=k} (\bm F\bm F^\top)_{ij}           && \text{\# property of trace}
    \end{flalign*}
	\vskip -0.2in \begin{flalign*}
		& \le \sum_{ij} \sqrt{\Var{\X_i|\Y=k}} \sqrt{\Var{\X_j|\Y=k}}\; (\bm F\bm F^\top)_{ij} & 
    \end{flalign*} \vskip -0.2in 
	\begin{flalign*}
		& = \sum_{ij} \sigma_i\sigma_j (\bm F\bm F^\top)_{ij}               &&\text{\# } \bm\sigma\in \mathbb{R}^m, \sigma_i\triangleq\sqrt{\Var{\X_i|\Y=k}} \\
        & = \bm\sigma^\top \bm F\bm F^\top \bm\sigma \\
        & \le \norm{\bm \sigma}^2_2 && \text{\# eigenvalues of $\bm F$ is no more than 1} \\
        & = \sum_i \Var{\X_i|\Y=k} \\
        & = \Var{\X|\Y=k}.
    \end{flalign*}
    Next, we prove the theorem with the above lemma.
    \begin{flalign*}
        &\quad \E{\Var{\bm F^\top\X|\Y}} & \\
        & = \sum_k \Pr(\Y=k) \Var{\bm F^\top\X|\Y=k} \\
        & \le \sum_k \Pr(\Y=k) \Var{\X|\Y=k} \\
        & = \E{\Var{\X|\Y}}
    \end{flalign*}
\end{proof}

\section{}\label{ap:theorem3}
\begin{theorem}
    If \;$\forall F_{ij}\not=0$, $\norm{\e_i-\e_j}_2\le\varepsilon$, then the distance between $\e_j$ and $\widehat{\e}_j = \sum_i F_{ij}\e_i$ is also less than or equal to $\varepsilon$, i.e.,
    {\normalfont
    \begin{equation*}
        \norm{\e_i - \e_j}_2\le \varepsilon,\;\forall F_{ij}\not=0
        \quad\Rightarrow\quad \norm{\e_j - \widehat{\e}_j}_2 \le \varepsilon,
    \end{equation*}}
    and $\varepsilon$ can be arbitrarily small with a proper $\bm F$.
\end{theorem}
\begin{proof}
    \begin{flalign*}
        &\quad \norm{\e_j - \widehat{\e}_j}_2
        = \norm{\e_j - \sum_i F_{ij}\e_i}_2 \\
        & = \norm{\sum_i F_{ij}(\e_j - \e_i)}_2                                  && \text{\# since $\sum_i F_{ij}=1$} \\
        & \le \sum_i F_{ij}\norm{\e_j - \e_i}_2                                  && \text{\# Cauchy-Schwarz inequality} \\
        & \le \sum_{i} F_{ij} \varepsilon = \varepsilon
    \end{flalign*}

Next, we prove that there exists such an $F$ that $\varepsilon$ is 0. This is equivalent to finding a doubly stochastic $F$ satisfying $\sum_i F_{ij}\e_i = \e_j$ for all $j$. Given trivial solution $F=I$, this equation is solvable. 
In most real-world attributed networks, the number of attributes is far greater than the number of classes, so the number of variables in this linear system is greater than the number of equations. Given that it is solvable, it must have infinite number of solutions other than $I$. Thus, $\varepsilon$ can be arbitrarily small with a proper $\bm F$.
\end{proof}

\end{document}